\documentclass{article} 
\usepackage{iclr2026_conference,times}
\usepackage{graphicx}
\usepackage{algorithm}
\usepackage{algorithmic}

\usepackage{amsmath,amsfonts,bm}

\def\eqref#1{equation~\ref{#1}}

\def\1{\bm{1}}

\DeclareMathAlphabet{\mathsfit}{\encodingdefault}{\sfdefault}{m}{sl}
\SetMathAlphabet{\mathsfit}{bold}{\encodingdefault}{\sfdefault}{bx}{n}

\DeclareMathOperator*{\argmax}{arg\,max}

\usepackage{caption}
\usepackage{graphicx}
\usepackage{caption}
\usepackage{booktabs}
\usepackage{multirow}
\usepackage{xcolor}
\usepackage{hyperref}
\usepackage{url}
\usepackage{amsthm}
\iclrfinalcopy
\newtheorem{proposition}{Proposition}[section]

\title{A Nash Equilibrium Framework for Training-Free Multimodal Step Verification}

\author{\textsuperscript{*}Rohit Sinha$^{1,2}$, \textsuperscript{*}Kunal Tilaganji$^{1,2}$, Tanuja Ganu$^{1}$, Nagarajan Natarajan$^{1}$, \And Amit Sharma$^{1}$, Vineeth N. Balasubramanian$^{1,2}$ \\
$^{1}$Microsoft Research India \quad $^{2}$Indian Institute of Technology Hyderabad \\
}

\begin{document}

\maketitle

\begin{abstract}

Multimodal large language models often generate reasoning chains containing subtle errors that lead to incorrect answers. Current verification approaches have notable limitations. Learned critics need extensive labeled data and show inconsistent performance across different tasks. Meanwhile, existing training-free methods simply average scores from different sources, missing a key insight: when these scores disagree, that disagreement itself carries important information about whether a reasoning step is truly valid or not. We propose a training-free verification approach that treats step-wise verification as a coordination problem among specialized judges. We formalize these judges' interaction as a Nash equilibrium game where agreement signals valid steps while disagreement reveals instability. Our method computes equilibrium scores through a closed-form solution, enabling both disagreement-aware filtering and stability-conscious ranking of reasoning steps. Evaluated across six benchmarks, our approach achieves consistent improvements of 2.4\% to 5.2\% over baseline models and shows competitive performance against learned critics, demonstrating that cross-modal agreement (not just average confidence) provides robust verification signals without task-specific adaptation.
\end{abstract}
\vspace{-5pt}

\section{Introduction}

Multimodal large language models (MLLMs) demonstrate impressive long-horizon reasoning over images and text~\cite{bai2025qwen3vltechnicalreport, NEURIPS2023_6dcf277e}, yet their reasoning chains often contain subtle errors, unsupported visual claims, or logical gaps that propagate to incorrect answers~\cite{li2023evaluating,liu2024hallucination}.

Process reward models address this by scoring individual steps rather than only final answers \citep{lightman2023prm800k,wang2023mathshepherd,luo2024omegaprm}. In multimodal settings, supervision pipelines and learned critics judge step correctness given visual context \citep{sun2025mmverify,wang2024visualprm}. These methods show that intermediate reasoning step verification lead to more reliable infernce.  However, a key bottleneck is that step-level verifications require carefully labeled training data, often obtained through Monte Carlo rollouts or human annotations \citep{luo2024omegaprm,wang2024visualprm} and empirical evidence also reveals tension in MLLM capabilities: grounding disrupts reasoning~\cite{kang2024vgent,zhang2023llavagrounding}, visual tuning degrades language performance~\cite{shen2024multimodal}, extended reasoning causes forgetting~\cite{chen2025morethought}, motivating a game-theoretic framework for step verification. 

Consider a setting where a visual verifier assigns 0.9 confidence to a step mentioning present objects, while a logical verifier assigns 0.2 due to incoherence. Their average (0.55) masks fundamental disagreement making the step appear "moderately acceptable". Standard aggregation~\cite{weaver,helpingherdingrewardmodel} treats all combinations equally, unanimous confidence (0.9, 0.9, 0.9) and conflicting evidence (0.9, 0.2, 0.9) yield similar averages. When independent judges disagree; this signals instability, and that step should not extend the reasoning chain.

This insight leads us to treat verification as a coordination problem between independent sources of evidence. Incorrect steps cause judge disagreement while correct steps yield convergence. Unlike variance (ignores belief strength) or entropy (treats (0.5, 0.5, 0.5) like (1.0, 0.0, 0.5)), the equilibrium framework captures how judges revise their beliefs in light of what others think, producing scores that balance collective confidence against remaining points of disagreement.

We formalize this as a training-free game among frozen, modality-specialized judges via Nash equilibrium~\citep{nash1950equilibrium}. This plug-in verifier requires no training or annotations, complementing or replacing learned critics. Accepted steps are those achieving stable consensus. Across six benchmarks spanning spatial reasoning and visual grounding, we achieve consistent improvements over learned critics and baseline models, demonstrating that disagreement structure provides robust, transferable verification signals.

\textbf{Our Contribution:} (i) We frame step-wise verification as coordination among multiple independent, frozen verifiers, establishing disagreement structure as a verification signal, and propose disagreement-aware verification as a design principle, rather than learning to predict step correctness (ii) We introduce a Nash equilibrium verifier with closed-form solution modeling cross-modal agreement. (iii) We provide a training-free, plug-in implementation integrating directly into MLLM reasoning without base model modification. To the best of our knowledge, this is a first effort towards inline verification in multimodal settings with multiple verifiers.

\vspace{-10pt}

\section{Related Work}
\label{sec:related_work}
\vspace{-10pt}
Modern vision-language models increasingly rely on explicit reasoning steps to handle compositional tasks, with benchmarks like MathVista, MMMU, and MathVision showing that structured chains of inference outperform direct prediction \citep{lu2023mathvista,yue2024mmmu} ~\cite{wang2024measuringmultimodalmathematicalreasoning}. Multimodal chain-of-thought prompting further shows how exposing intermediate steps improves both interpretability and accuracy \citep{zhang2024multimodal}. While these approaches excel at generating reasoning traces, evaluating their correctness remains challenging. Process reward models (PRMs) address this by scoring intermediate reasoning states rather than just final answers, improving sample efficiency and stability \citep{cobbe2021training} ~\cite{lightman2023letsverifystepstep}. Recent multimodal extensions \cite{mmprm,visualprm} adapt this paradigm using synthetic rollouts and large-scale training to assess step correctness in visual contexts \citep{luo2024omegaprm}. Separately, learned critic models \cite{sherlock,criticv} judge reasoning quality by detecting hallucinations and checking grounding, while \citep{sun2025mmverify} shows that frozen large models can verify chain-of-thought reasoning post hoc without retraining. However, these approaches typically produce a singular scalar score that implicitly aggregate heterogeneous evidence without explicitly modeling disagreement structure, though ensemble learning and multi-agent reasoning literature suggest disagreement itself can be highly informative \citep{NIPS2017_9ef2ed4b,du2023debate,nash1950equilibrium}. Our equilibrium based verifier builds on these insights by framing step-wise verification as a coordination problem among frozen, modality-specialized judges. Instead of learning to predict step correctness, it asks whether independent evaluators can reach a stable agreement given their respective evidence. This formulation emphasizes stability and coherence of reasoning steps rather than confidence alone.

\section{Nash-Equilibrium–Based Step-wise Verification}
\label{sec:nash_verifier}

\textbf{Verification Setting:} At a given reasoning step \(t\), the base model is provided with an image \(I\), a question \(Q\), and a partial reasoning trace \(r_{1:t-1}\). A base MLLM generates \(n\) candidate step \(r_t\), which must be evaluated before extending the reasoning chain. The verifier returns a binary decision: accept the step and continue reasoning, or reject it and resample. The evaluation is local to the step under consideration and does not require access to future reasoning or the final answer. This locality allows errors to be intercepted at the point where they arise, rather than after they have propagated.

\textbf{Verifier Agents: } Verification is carried out by a small set of frozen MLLMs, each prompted to judge the step from a distinct perspective. In the experiments presented later, three agents are used:
(i) \textbf{Visual Agent (V)}: a visual verifier that assesses whether the step is supported by the image and is visually verifiable,
(ii) \textbf{Logical Agent (L)}: a logical verifier that evaluates whether the reasoning step follows logically from previous steps and progresses toward answering the question, and
(iii) \textbf{Contextual Agent (C)}: a contextual  verifier that assesses whether the step maintains focus on the original question and avoids introducing irrelevant information.

Each agent is prompted to output a single scalar score $\hat{s}_i \in [0,1]$, interpreted as its subjective confidence that the reasoning step is valid, given its modality-specific evidence. The prompts enforce a fixed output format. Agents operate in complete isolation. They never see other agents' scores, and their prompts contain no information about other agents' judgments. No fine-tuning or calibration is performed.

\textbf{Agreement Game:} Rather than aggregating the raw scores directly, verification is framed as a coordination problem among the agents. The key intuition is, if a reasoning step is truly valid, independent judges evaluating this step from a different perspectives should be able to reach agreement about it. If they cannot agree even after accounting for each other's perspectives, the step is likely unstable. We model this as a game where each agent chooses to softly adjust its initial belief toward the group consensus, but only to the extent that doing so does not overly contradict its own evidence. 
The Nash equilibrium of this game tells us the final scores after all agents have implicitly accounted for the possibility that others might have valid reasons for disagreeing.

Each agent selects a reported score \(s_i \in [0,1]\), balancing agreement with others against fidelity to its own judgment. The interaction is modeled as a quadratic game in which the payoff to agent \(i\) is
$u_i(s_i, s_{-i})
=
- (s_i - \bar{s}_{-i})^2
- \lambda_i (s_i - \hat{s}_i)^2,$
where \(\bar{s}_{-i}\) denotes the mean reported score of the remaining agents and \(\lambda_i > 0\) controls agent \(i\)'s strength of self-consistency. The \(\lambda_i\) values encode how resistant each agent should be to consensus pressure. A higher \(\lambda_i\) means agent \(i\) trusts its own modality-specific evidence more strongly and will deviate less from its raw score even when others disagree. This formulation encourages consensus while penalizing excessive deviation from the agent’s original belief. The game is strictly concave in each player’s strategy and admits a unique Nash equilibrium (see Appendix~\ref{app:proof} for proof).

\textbf{Equilibrium Computation: } At equilibrium, each agent’s reported score satisfies
$ s_i^* = \frac{\bar{s}_{-i}^* + \lambda \hat{s}_i}{1 + \lambda}.$
The system of equations has a closed-form solution and can be computed directly from the raw scores \(\{\hat{s}_i\}\). No iterative solver or learning-based optimization is required, and equilibrium computation adds negligible overhead relative to the verifier queries themselves.

A critical property of the equilibrium is that it preserves the mean while dampening disagreement. The mean equilibrium score equals the mean raw score, but individual scores converge as disagreement diminishes proportionally to inter-agent conflict. This matters because it allows us to separate two failure modes: (1) steps with low average confidence (collective doubt), and (2) steps with high average confidence but high dispersion (conflicting evidence). A simple average would accept both (0.6, 0.6, 0.6) and (0.9, 0.3, 0.6) equally, but equilibrium dispersion reveals that the second case reflects unresolved instability. The complete mathematical formulation and implementation details are provided in Appendix~\ref{app:implementation}.

The equilibrium computation adds negligible overhead. Verification cost is dominated by model inference. Empirically, our method requires $3.80$ the wall-clock time of the base model, closely matching the theoretical prediction of $3.27$ Appendix \ref{app:comp_complexity} has detailed complexity analysis.

\textbf{Acceptance Criterion: } Once equilibrium scores \(\{s_i^*\}\) are obtained, two summary statistics are computed:
the mean confidence $\bar{s}^* = \frac{1}{N}\sum_i s_i^*,$
and the dispersion $\Delta^* = \frac{1}{N}\sum_i |s_i^* - \bar{s}^*|.$

A reasoning step is accepted if the mean confidence exceeds a fixed threshold \(\tau\) and the dispersion is below a tolerance \(\epsilon\). This dual criterion enables two complementary verification mechanisms. The dispersion check $(\Delta < \epsilon)$ filters candidates with conflicting cross-modal evidence, while the confidence check $(\bar{s}^* > \tau)$ ensures
sufficient collective endorsement. Among accepted candidates, ranking by$ \bar{s}^*$ selects the most stable continuation. When no candidates satisfy both criteria, continuous ranking $(\bar{s}^* - \Delta)$ provides fallback.

\vspace{-10pt}
\subsection{Experimental Setup}
\vspace{-10pt}
We evaluate our verification framework on six benchmarks that collectively assess different aspects of visual and multimodal reasoning. These include 3DSRBench \cite{3dsrbench}, CV-Bench-2D  \cite{cvbench}, AI2D \cite{ai2d} MMStar \cite{mmstar}, BLINK \cite{blink} for spatial understanding, perception-heavy visual question answering, and general multimodal reasoning capabilities. Our approach employs three specialized verification agents: Visual (V), Logical (L), and Contextual (C), built on Qwen2.5-VL-7B-Instruct, to collaboratively verify reasoning steps generated by Qwen2.5-VL-7B-Instruct. Each agent operates with distinct prompts and scoring criteria. Detailed experimental setup is presented in Appendix \ref{app:experimental_setup}, and equilibrium computation details are in Appendix \ref{app:implementation}

At each reasoning step, the base model generates three candidate continuations via temperature sampling ($T=0.8$, top-$p=0.6$). Each agent provides raw scores $\hat{s}_i \in [0,1]$ for every candidate, which we convert to equilibrium scores $s_i^*$ using agent-specific stubbornness parameters ($\lambda_V = 1.5$, $\lambda_L = 1.0$, $\lambda_C = 0.8$). A candidate step is accepted if its equilibrium dispersion $\Delta < 0.1$ and mean score $\bar{s}^* > 0.6$. Among accepted steps, we select the one with the highest $\bar{s}^*$. See Appendix~\ref{app:experimental_setup} for detailed agent prompts.

\section{Main Results}
\begin{table}[!h]
\centering

\begin{tabular}{lcccccc}
\hline
\textbf{Dataset} 
& \textbf{3DSRBench} 
& \textbf{CV-Bench-3D} 
& \textbf{CV-Bench-2D} 
& \textbf{BLINK} 
& \textbf{MMStar} 
& \textbf{AI2D} \\
\hline
Base 
& 56.12 
& 76.39 
& 49.27 
& 42.86 
& 57.25 
& 76.52 \\
LLaVA Critic 
& 52.71 
& 81.58 
& 67.52 
& 45.50 
& 61.07 
& 76.61 \\
Sherlock 
& 11.50 
& 28.67 
& 68.78 
& 39.61 
& 54.20 
& 61.54 \\
VisionSR1 
& 35.50 
& 34.28
& 53.82 
& 30.09 
& 57.20 
& 60.32 \\
\textbf{Nash (Ours)}
& \textbf{59.02} 
& \textbf{82.34} 
& \textbf{71.51} 
& \textbf{46.15} 
& \textbf{63.21} 
& \textbf{78.95} \\
\hline
\end{tabular}
\caption{Accuracy (\%) across multimodal reasoning benchmarks under different step-wise verification strategies}

\label{tab:main_results}
\end{table}

\vspace{-15pt}
 Table~\ref{tab:main_results} compares our approach against the base model and verification baselines, where the key takeaway is the consistent performance gains of our verifier across all benchmarks, demonstrating robust improvements across diverse reasoning settings.

\textbf{Finding 1: Consistent gains across diverse reasoning tasks and on hallucination-prone tasks.} 
Nash verifier shows positive results on all six benchmarks (+2.4\% to +22.2\%) without task-specific adaptation. Gains span spatial reasoning (3DSRBench: +2.9\%, CV-Bench-3D: +6.0\%), visual grounding (CV-Bench-2D: +22.2\%, BLINK: +3.3\%), and abstract reasoning (MMStar: +6.0\%, AI2D: +2.4\%), suggesting equilibrium verification may capture cross-modal agreement patterns. Larger gains observed on CV-Bench-2D (+22.2\%), known for fluent but weakly grounded reasoning. BLINK shows +3.3\%. Equilibrium dispersion appears to detect steps where linguistic confidence may lack visual/contextual support, potentially reducing error propagation from plausible-sounding but unstable steps.

\textbf{Finding 2: Training-free verification shows competitive results.}
Nash verification achieves comparable performance without training data. On CV-Bench-2D (+22.2\%), it performs favorably against LLaVA Critic (+18.3\%) and Sherlock (+19.5\%). Unlike learned critics showing variability (LLaVA degrades 3DSRBench -3.4\%), equilibrium verification suggests reliance on structural cross-modal agreement properties. \textbf{Learned critics show task-specific variability.}
VisionSR1 and Sherlock exhibit substantial variation: VisionSR1 degrades 3DSRBench (-20.6\%), AI2D (-16.2\%); Sherlock declines on 3DSRBench (-44.6\%), CV-Bench-3D (-47.7\%), AI2D (-15.0\%). Nash verification maintains positive results across all tasks, avoiding degradation below baseline on tested benchmarks.

\textbf{Finding 3: Results on high-performing base models.}
Nash verification adds +2.4\% on AI2D despite 76.5\% base accuracy, which suggests the proposed method may function as a stability regularizer, encouraging reasoning trajectory coherence rather than only detecting obvious errors.

\textbf{Finding 4: Spatial reasoning reveals critic sensitivity.}
On 3DRSRBench, equilibrium verification improves accuracy from 56.12\% to 59.02\%, while LLaVA Critic and VisionSR1 perform below baseline, suggesting learned critics may be sensitive to spatial reasoning patterns where cues conflict. Our approach treats disagreement as information and addresses conflicting perspectives through equilibrium.

\vspace{-10pt}




\section{Conclusion and Takeaways}
\vspace{-10pt}

Multimodal reasoning often fails gradually rather than catastrophically. Individual steps appear locally plausible creating trajectories that drift toward incorrect answers. Our experiments reveal that the disagreement structure among specialized verifiers exposes this instability more reliably than confidence scores alone, providing an early warning signals for unstable reasoning steps. By treating verification as coordination rather than classification, our training-free approach achieves consistent improvements across diverse benchmarks without task-specific tuning. The method's effectiveness depends fundamentally on maintaining diverse agent perspective that reveals conflicting evidence before errors propagate. Limitations include reliance on frozen model quality. Future work should explore dynamic agent selection and adaptive equilibrium parameters.


\bibliography{references}
\bibliographystyle{iclr2026_conference}

\newpage
\section{Appendix: Table of Contents}
\appendix


\begin{itemize}
    \item \textbf{A. HyperParameter}
    \item \textbf{B. Ablation}
    \item \textbf{C. Computational Complexity Analysis]}
    \item \textbf{D. Detailed Experimental Setup}
    \item \textbf{E. Implementation Details: Nash-Equilibrium Computation}
    \item \textbf{F. Nash-Equilibrium Existence and Uniqueness}
    \item \textbf{G. Qualitative Samples}
    \item \textbf{H. Prompt Templates}
    
\end{itemize}


\section{Hyperparameters}

\begin{table}[!h]
\centering
\begin{tabular}{l c}
\hline
Hyperparameter & Value \\
\hline
$\lambda_{\text{V}}$        & $1.5$ \\
$\lambda_{\text{L}}$    & $1.0$ \\
$\lambda_{\text{C}}$    & $0.8$ \\
$n$                     & 5 \\
$\tau$                  & 0.6\\
$\epsilon$              &  0.1 \\
\hline
\end{tabular}
\caption{Hyperparameters used in all experiments.}
\label{tab:lambda_values}
\end{table}

In our experiments, we used hyperparameters as shown in Table \ref{tab:lambda_values}. This configuration reflects the intuition that the visual verifier (V) should be most resistant to consensus pressure on perception-heavy steps, while the contextual verifier (C) may be more flexible when visual or linguistic evidence is strong. These values are fixed across all datasets and require no tuning.
\section{Abalation}
We perform ablation studies for $\tau$ and $\epsilon$ sensitivty  on 50 randomly sampled instances, which efficiently isolates component contributions while maintaining representative coverage of task diversity. Though full-dataset analysis could further refine these insights, the observed trends provide clear evidence of each component's role in verification performance.

\subsection{Threshold Sensitivity and Equilibrium as Continuous Ranking}
\label{ab:tau_sweep}
\begin{table}[h]
\centering
\small
\begin{tabular}{lcccccl}
\hline
$\tau$ & Accuracy   & Accept & Mean $\bar{s}^*$ & Mean $\Delta$ & Selection Mode \\
\hline
0.0001 & 68.63\% & 85.1\% & 0.971 & 0.009 & Normal (89\%) \\
0.001  & 66.67\%  & 87.4\% & 0.974 & 0.008 & Normal (90\%) \\
0.01   & 64.71\%   & 86.3\% & 0.970 & 0.009 & Normal (90\%) \\
0.1    & 60.78\%  & 83.2\% & 0.978 & 0.007 & Normal (89\%) \\
0.6    & 60.78\%    & 84.4\% & 0.968 & 0.010 & Normal (89\%) \\
1.0    & 60.78\%    & 0.0\%  & ---   & ---   & Fallback (100\%) \\
10.0   & 70.59\%   & 0.0\%  & ---   & ---   & Fallback (100\%) \\
\hline
\end{tabular}
\vspace{-2mm}
\caption{Confidence threshold ablation on 3DSRBench with dispersion tolerance 
fixed at $\epsilon = 0.1$. Base model accuracy: 56.12\%. Refer to Algorithm \ref{alg:nash_verification} for the selection modes in the proposed method.}
\label{tab:tau-sweep-abtable}
\end{table}

\begin{figure}[h]
    \centering
    \includegraphics[width=0.9\linewidth]{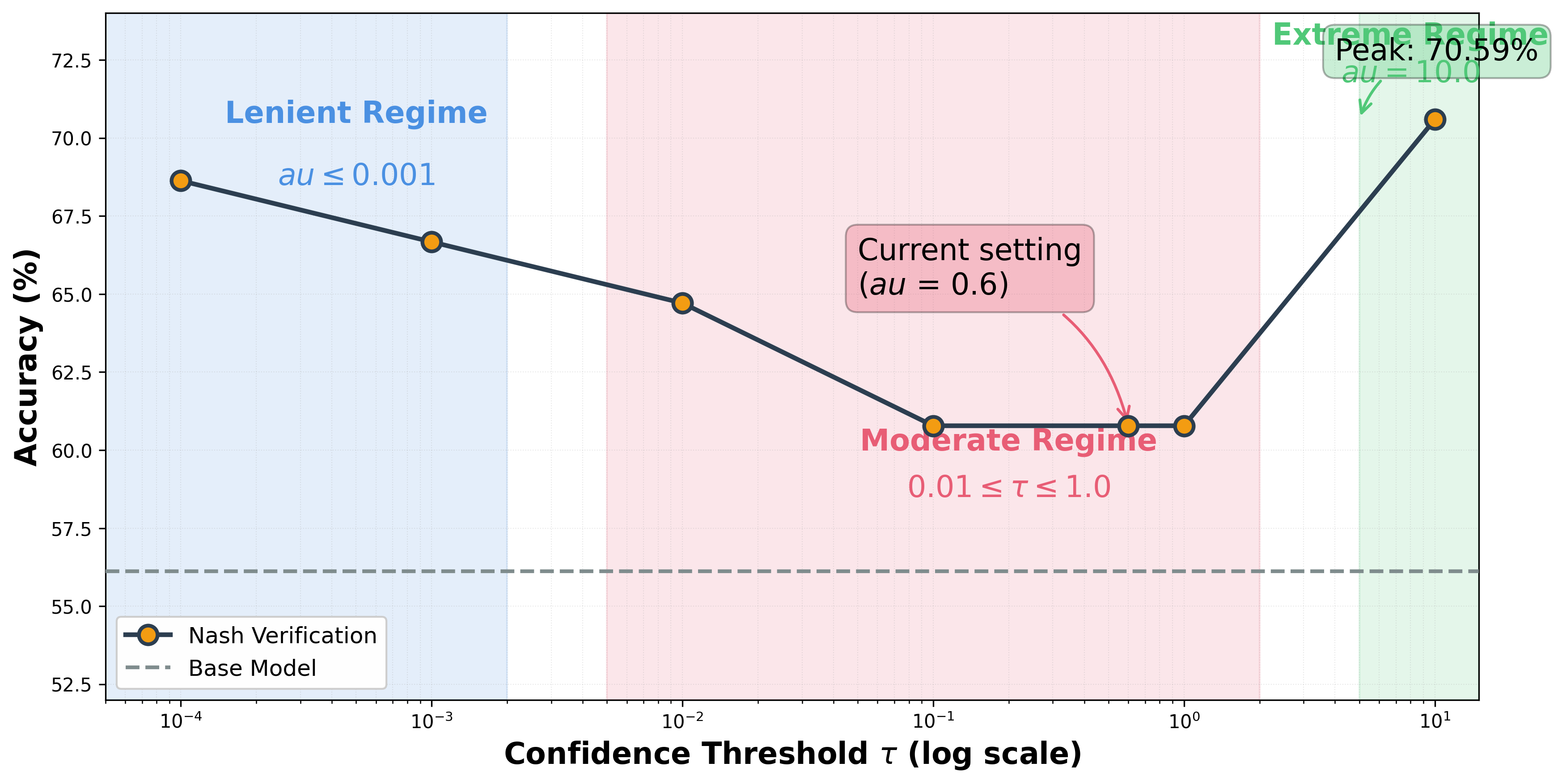}
    \caption{Threshold sensitivity on 3DSRBench with $\epsilon = 0.1$ fixed. Performance exhibits a U-shaped curve with highest accuracy at $\tau = 10.0$ (70.6\%), where the acceptance criterion is never satisfied and selection relies entirely on continuous equilibrium ranking $\text{argmax}(\bar{s}^* - \Delta)$. Intermediate thresholds ($\tau$ = 0.1-0.6) show the lowest performance by filtering out borderline cases with informative disagreement patterns. The dashed line indicates base model accuracy (56.12\%). The curve's shape reveals three distinct operating regimes: permissive ranking (left), selective filtering (bottom), and continuous stability ranking (right).}
    \label{fig:ushapetau}
\end{figure}

To understand how the acceptance criterion affects verification quality, we ablated the confidence threshold $\tau$ while holding the dispersion tolerance fixed at $\epsilon = 0.1 $
on 3DSRBench. We tested $\tau \in \{10^{-4}, 10^{-3}, 10^{-2}, 10^{-1}, 0.6, 1.0, 10.0\}$, spanning permissive to restrictive settings.

Table \ref{tab:tau-sweep-abtable} and Figure \ref{fig:ushapetau} shows a surprising U-shaped performance curve. Very permissive thresholds $(\tau \leq 0.01)$ achieves 65-69\% accuracy by accepting 85-87\% of candidate steps and selecting among them based on equilibrium scores. Intermediate values ($\tau$ = 0.1 0.6)  show lower performance at around 61\%, despite similar acceptance rates of 83-86\%.  Most unexpectedly, $\tau$ = 10.0 achieves the highest accuracy at 70.6\%. Yet this setting accepts zero candidates, as no equilibrium mean score exceeds 10.0.

This U-shaped curve reveals a fundamental insight about how equilibrium scores 
encode verification information. The curve's shape reflects three distinct operating 
regimes, each telling us something different about the nature of cross-modal agreement.

\textbf{Left side of the U $(\tau \leq 0.01)$: Permissive regime.} Here the threshold is so low that nearly all candidates pass, and performance depends primarily on how well the equilibrium scores rank accepted candidates. The strong performance (65-69\%) indicates that even among mostly accepted candidates, the relative ordering by $\bar{s}^*$ captures meaningful quality differences.

\textbf{Bottom of the U $(\tau = [0.1 ,0.6])$: Selective filtering regime.} This is where the U-shape becomes particularly revealing. These thresholds accept 83-84\% of candidates, slightly more selective than the permissive regime, yet performance drops by 4-8 percentage points. What's being filtered out? The borderline cases where judges partially disagree but haven't reached a strong consensus. For spatial reasoning, these are exactly the cases where cross-modal conflict is most informative: when visual evidence suggests one answer but linguistic priors or logical consistency pull toward another. 

\textbf{Right side of the U $(\tau \geq 1.0)$: Continuous ranking regime} When $\tau$ = 10.0 rejects all candidates (since equilibrium scores are bounded in [0,1]), the system necessarily falls back to selecting $\text{argmax}(\bar{s}^* - \Delta)$. This achieves the highest performance (70.6\%) because it preserves the full equilibrium information. Both collective confidence ($\bar{s}^*$) and disagreement structure ($\Delta$) contribute to every decision. In this case, continuous ranking via $\text{argmax}(\bar{s}^* - \Delta)$, preserves full equilibrium information about both confidence and disagreement structure.

The U-shaped pattern in threshold sensitivity tells us something important about how Nash equilibrium scores actually work best. These scores function naturally as a ranking tool, not as a binary classifier. That's why both extremes of the threshold spectrum perform well: they both respect the underlying ranking structure. At the permissive end (low $\tau$), we're essentially ranking candidates by their equilibrium-adjusted confidence score $\bar{s}^*$ alone. At the restrictive end (high $\tau$), we rank by $\text{argmax}(\bar{s}^* - \Delta)$, where disagreement acts as a tiebreaker that penalizes borderline cases. Both approaches succeed because they use the scores to compare and order candidates. The middle ground fails precisely because it tries to do something different. It converts these ranking-optimized scores into binary accept/reject decisions. Here, disagreement becomes grounds for outright rejection rather than information that helps us choose between options. This matters especially for spatial reasoning tasks, where reconciling naturally creates disagreement among specialized evaluators (V,L,C). That disagreement carries useful signal about relative quality when we're ranking steps against each other, but it becomes a crude filter when we're just trying to classify steps as correct or incorrect. The curve validates two insights that work together: (i) first, equilibrium coordination outperforms simple score averaging even when applied permissively across the board. (ii) Second, incorporating disagreement structure $(\bar{s}^* - \Delta)$ provides additional discriminative power when used universally for ranking.

\subsection{Epsilon Sensitivity Analysis}

We investigate how the dispersion tolerance $\epsilon$ affects verification quality on 3DSRBench spatialreasoning tasks. Epsilon controls when the Nash equilibrium iteration considers judges to have reached sufficient agreement. We evaluate seven epsilon values: $\epsilon \in  \{0.001, 0.05, 0.1, 0.5, 1.0, 2.5, 3.0\}$, spanning from very strict (requiring near-perfect agreement) to very permissive (accepting substantial residual disagreement). Keeping $\tau = 0.6$ fixed 

Table \ref{tab:epsilon-sensitivity} and Figure \ref{fig:epsabs} present accuracy across epsilon configurations.
\begin{table}[h]
\centering
\begin{tabular}{lcccccc}
\hline
$\varepsilon$ & Accuracy & Accept Rate & Mean $\bar{s}^*$ & Mean $\Delta$ \\
\hline
0.001 & 54.75\% & 66.29\% & 0.9227 & 0.0253  \\
0.05  & 56.71\% & 84.19\% & 0.8869 & 0.0261\\
0.1   & 60.63\% & 84.07\% & 0.8811 & 0.0293 \\
0.5   & 61.00\% & 86.91\% & 0.8820 & 0.0285\\
1.0   & 60.63\% & 92.28\% & 0.9224 & 0.0235\\
2.5   & 58.17\% & 89.18\% & 0.8950 & 0.0235\\
3.0   & 57.67\% & 91.69\% & 0.9257 & 0.0221\\
\hline
\end{tabular}
\caption{Epsilon sensitivity analysis on 3DSRBench. Accept Rate indicates the percentage of candidate steps satisfying both the dispersion criterion and confidence threshold. Mean $\bar{s}^*$ and Mean $\Delta$ represent equilibrium statistics averaged across all candidate steps evaluated during verification.}
\label{tab:epsilon-sensitivity}
\end{table}

\begin{figure}[h]
    \centering
    \includegraphics[width=0.9\linewidth]{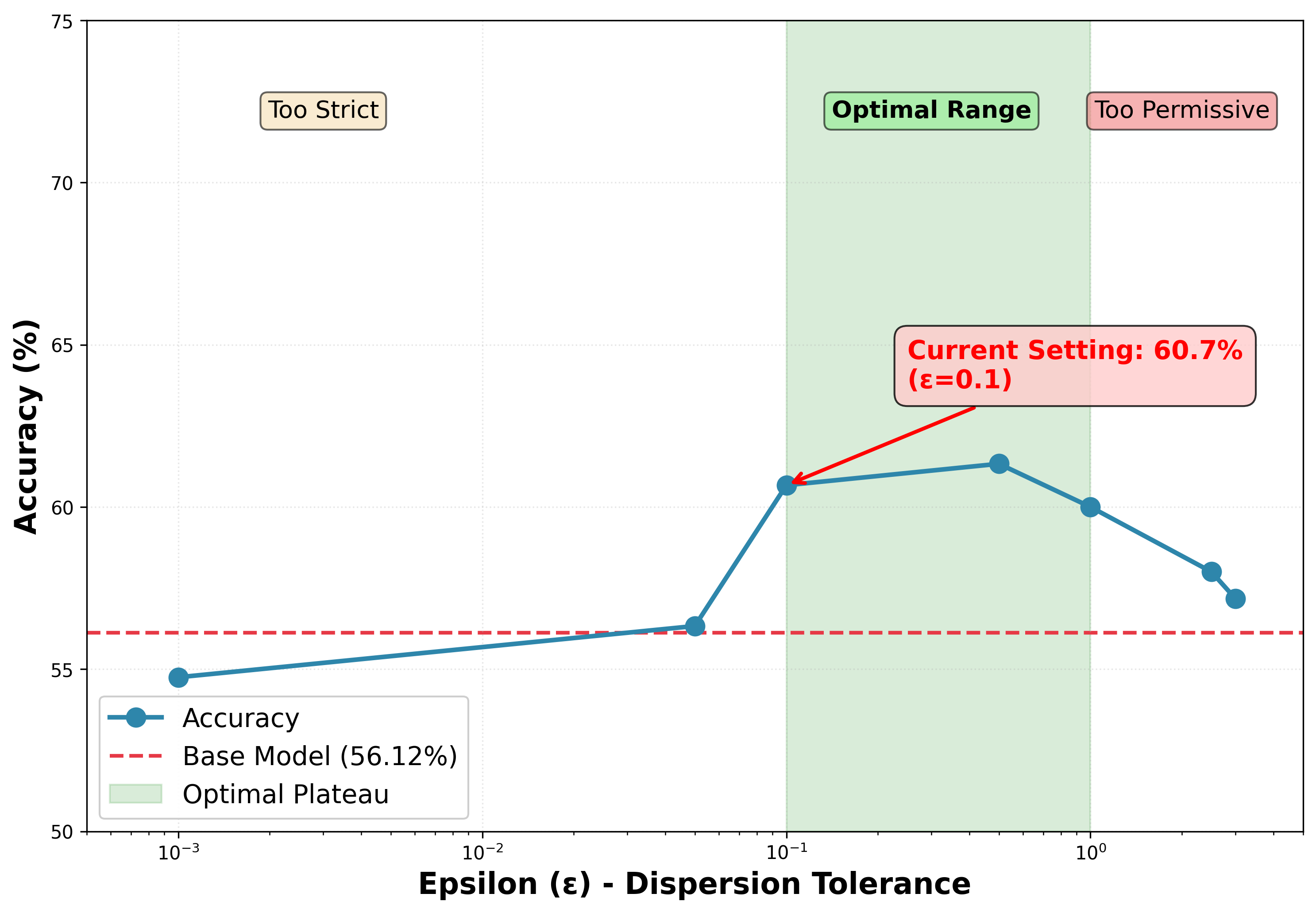}
    \caption{Epsilon sensitivity on 3DSRBench. The plot shows accuracy versus dispersion tolerance $\epsilon$ }
    \label{fig:epsabs}
\end{figure}

\textbf{Optimal plateau region:} Performance peaks in the range $\epsilon \in [0.1, 1.0]$, with accuracy hovering around 60-61\%. This plateau represents the "sweet spot" where dispersion filtering provides meaningful discrimination, rejecting unstable candidates while accepting those where judges have genuinely converged without being overly restrictive.

\textbf{Overly strict regime ($\epsilon < 0.1$):} Very low epsilon values (0.001, 0.05) show 
degraded performance (54-56\%). These settings demand near-unanimous agreement, which may be unrealistic for spatial reasoning tasks. When $\epsilon$ is too strict, many legitimate reasoning steps get rejected for having $\Delta $ slightly above the threshold, even though the underlying disagreement has been substantially resolved through equilibrium adjustment. Performance doesn't collapse entirely because the system falls back to ranking mode (selecting $\text{argmax}(\bar{s}^* - \Delta)$ when no candidates pass acceptance). This fallback mechanism, as discussed in Section \ref{ab:tau_sweep}, proves effective, which explains why even $\epsilon = 0.001$ maintains reasonable performance.

\textbf{Overly permissive regime ($\epsilon  > 1.0$)}: As epsilon increases beyond 1.0, we observe gradual performance degradation (60.63\% at $\epsilon = 1.0$, declining toward 57.67-58.17\% at $\epsilon = [2.5,3.0]$). This pattern aligns with theoretical expectations. When $\epsilon$ is too large, the dispersion filter loses discriminative power, accepting candidates where judges haven't truly reached a stable consensus. The decline happens gradually rather than all at once because candidates must pass two tests simultaneously: $\Delta < \epsilon$ \textbf{and} $\bar{s}^* > \tau$. Even when we relax the dispersion requirement, weak candidates are still caught by the confidence threshold of $\tau = 0.6$. The problem is more subtle: with looser dispersion filtering, the system loses its ability to tell the difference between ``strong agreement with high confidence'' ($\bar{s}^* = 0.9$, $\Delta = 0.05$) and ``false consensus with hidden disagreement'' ($\bar{s}^* = 0.9$, $\Delta = 0.8$). This allows problematic steps to pass through undetected.

\textbf{Relationship to Nash equilibrium theory}: This pattern validates a core intuition of our formulation. The Nash equilibrium doesn't just average scores, it represents a stable state where each judge, knowing what others believe, has no incentive to deviate from their reported score. 

When $\epsilon$ is appropriately calibrated (0.1-1.0), we're effectively asking: "Have the judges converged to a configuration where everyone is sufficiently satisfied given what others think?" This captures genuine cross-modal agreement. When $\epsilon$ is too strict, we reject steps where substantial agreement has been reached but perfect unanimity hasn't. When$\epsilon$ is too permissive, we accept steps where judges are still in meaningful disagreement, indicating the reasoning step may be unstable.



\subsection{Ablation Study: Rejection vs Selection Mechanisms} 

To understand whether our improvements come mainly from filtering out unreliable candidates or from better ranking among viable options, we break down our method into five distinct strategies. \textbf{Full Nash} represents our complete approach: we first filter candidates using the acceptance criterion ($\Delta < \epsilon \land \bar{s}^* > \tau$), then select the option with the highest confidence via $\argmax(\bar{s}^*)$. \textbf{No Rejection} removes the filtering step entirely. All three candidates pass through at each step, with selection based solely on continuous ranking $\argmax(\bar{s}^* - \Delta)$. This lets us isolate how much the selection mechanism contributes by itself, without any quality gate. Conversely, \textbf{No Selection} keeps the same acceptance criterion as Full Nash but randomly picks among the accepted candidates rather than ranking them. This isolates the impact of the filtering mechanism alone. \textbf{Raw Average} tests whether the Nash equilibrium adjustment adds value by replacing equilibrium scores with simple means and dispersions of raw verifier outputs, while keeping the same dual-criterion filtering and argmax selection structure. Finally, \textbf{Random} serves as our baseline. It accepts everything and chooses randomly, representing a system with neither intelligent filtering nor ranking. By systematically turning each component on and off, we can determine which specific mechanisms drive performance improvements rather than attributing results to their combined effects.

\subsection{Mechanism Decomposition Across Datasets} 
Table \ref{tab:ablation_mechanisms} presents our ablation analysis. We report two metrics across four benchmarks: \textit{Acc.} indicates the average number of candidates accepted per reasoning step (out of 5 generated), while \textit{Rej.\%}  shows the percentage of steps where all candidates are rejected, requiring fallback to continuous ranking. Each row isolates a different aspect of the verification process. 

The comparison between \textbf{Full Nash} and \textbf{No Rejection} reveals what the filtering mechanism accomplishes. When we accept every candidate, the pool size jumps from roughly 2.5 to 3.0, meaning rejection weeds out 12-17\% of options. The \textbf{No Selection} baseline keeps this same filtering but randomly picks from accepted candidates instead of ranking by equilibrium confidence. This lets us see what the selection mechanism adds. Interestingly, Full Nash and No Selection show identical acceptance numbers (both land at 2.64 on CV-Bench-2D, for instance), which confirms that selection works downstream of filtering; it chooses among candidates that already passed, rather than changing which ones pass in the first place.The \textbf{Raw Average} baseline swaps out Nash equilibrium scores for straightforward mean and dispersion calculations. Here we see something noteworthy: the equilibrium-based approach accepts 0.14-0.24 more candidates per step while needing fallback 3-7 percentage points less often. This dual advantage suggests that the Nash equilibrium offers better calibrated disagreement detection. Rather than treating all variance in confidence the same way, it distinguishes between genuine cross-modal conflicts (which should trigger rejection) and ordinary fluctuations in confidence levels (which shouldn't). The result is a system that's simultaneously more permissive and more reliable.

\begin{table}[h]
\centering
\caption{Ablation analysis decomposing contributions of rejection (filtering) and selection (ranking) mechanisms. Values show average candidates per step and percentage of steps requiring fallback.}
\label{tab:ablation_mechanisms}
\small
\begin{tabular}{@{}lcccccccc@{}}
\toprule
\multirow{2}{*}{\textbf{Strategy}} & \multicolumn{2}{c}{\textbf{CV-Bench-2D}} & \multicolumn{2}{c}{\textbf{CV-Bench-3D}} & \multicolumn{2}{c}{\textbf{3DSRBench}} & \multicolumn{2}{c}{\textbf{AI2D}} \\
\cmidrule(lr){2-3} \cmidrule(lr){4-5} \cmidrule(lr){6-7} \cmidrule(lr){8-9}
& Acc. & Rej.\% & Acc. & Rej.\% & Acc. & Rej.\% & Acc. & Rej.\% \\
\midrule
\textbf{Full Nash} & \textbf{2.64} & \textbf{9.1} & \textbf{2.84} & \textbf{2.3} & \textbf{2.54} & \textbf{12.0} & \textbf{2.50} & \textbf{13.4} \\
(Paper's method) & & & & & & & & \\
\midrule
No Rejection & 3.00 & 0.0 & 3.00 & 0.0 & 3.00 & 0.0 & 3.00 & 0.0 \\
(Selection only) & \textcolor{red}{+0.36} & \textcolor{blue}{-9.1} & \textcolor{red}{+0.16} & \textcolor{blue}{-2.3} & \textcolor{red}{+0.46} & \textcolor{blue}{-12.0} & \textcolor{red}{+0.50} & \textcolor{blue}{-13.4} \\
\midrule
No Selection & 2.64 & 9.1 & 2.84 & 2.3 & 2.54 & 12.0 & 2.50 & 13.4 \\
(Rejection only) & \textcolor{blue}{-} & - & - & - & - & - & - & - \\
\midrule
Raw Average & 2.40 & 16.2 & 2.70 & 5.4 & 2.31 & 17.2 & 2.35 & 17.6 \\
(No Nash Eq.) & \textcolor{blue}{-0.24} & \textcolor{red}{+7.1} & -0.14 & +3.1 & \textcolor{blue}{-0.23} & \textcolor{red}{+5.2} & \textcolor{blue}{-0.15} & \textcolor{red}{+4.2} \\
\midrule
Random & 3.00 & 0.0 & 3.00 & 0.0 & 3.00 & 0.0 & 3.00 & 0.0 \\
(Baseline) & +0.36 & -9.1 & +0.16 & -2.3 & +0.46 & -12.0 & +0.50 & -13.4 \\
\bottomrule
\end{tabular}

\vspace{0.3cm}
\small
\textit{Acc.}: Average number of accepted candidates per step (out of 3). \\
\textit{Rej.\%}: Percentage of steps where all candidates rejected (requiring fallback). \\
Values in \textcolor{red}{red} indicate degradation, \textcolor{blue}{blue} indicate improvement relative to Full Nash.
\end{table}

\subsection{Key Findings}

\paragraph{Finding 1: Rejection Mechanism Contribution}
The rejection mechanism filters out 12-17\% of candidates on average (from 3.00 to 2.50-2.84 accepted per step). Critically, in 2.3-13.4\% of reasoning steps, \textit{all} candidates are rejected, forcing the system to fall back to continuous ranking. This suggests the acceptance criterion successfully identifies steps where cross-modal disagreement signals instability.

\paragraph{Finding 2: Selection Mechanism Contribution}
Comparing "Full Nash" vs "No Selection" (which uses random choice among accepted candidates) reveals that both strategies accept the same number of candidates (2.50-2.84) and have identical fallback rates. However, the Full Nash method's intelligent ranking among these accepted candidates selected by the highest equilibrium confidence $\bar{s}^*$ enables better performance. This demonstrates that \textit{among acceptable steps}, the ranking provided by equilibrium scores effectively identifies the most stable continuation.

\paragraph{Finding 3: Complementary Mechanisms}
The rejection and selection mechanisms are complementary rather than redundant:
\begin{itemize}
    \item \textbf{Rejection} (binary filter): Prevents clearly problematic steps from propagating. Acts as a quality gate based on a disagreement structure.
    \item \textbf{Selection} (ranking): Optimizes among plausible candidates. Chooses the option with the highest collective confidence.
\end{itemize}

\paragraph{Finding 4: Nash Equilibrium vs Raw Averaging}
Comparing "Full Nash" to "Raw Average" isolates the contribution of equilibrium-adjusted scores. Raw averaging results in:
\begin{itemize}
    \item Fewer accepted candidates (2.31-2.70 vs 2.50-2.84)
    \item Higher fallback rates (5.4-17.6\% vs 2.3-13.4\%)
    \item More aggressive rejection (difference of 3.1-7.1 percentage points)
\end{itemize}
This indicates that Nash equilibrium scores provide more nuanced disagreement detection, distinguishing between genuine instability (high dispersion) versus simple variance in confidence levels, enabling better calibration of the acceptance criterion.

\subsection{Implications for Framework Design}

The ablation analysis reveals that the Nash equilibrium framework provides value through \textit{two distinct pathways}:

\textbf{Path 1: Disagreement-Aware Filtering.} Equilibrium dispersion ($\Delta = \frac{1}{N}\sum_i |s^*_i - \bar{s}^*|$) captures cross-modal conflict patterns that simple averaging misses. The acceptance criterion ($\Delta < \epsilon \land \bar{s}^* > \tau$) acts as a quality gate, filtering out candidates where specialized judges fundamentally disagree about validity.

\textbf{Path 2: Stability-Aware Ranking.} Among candidates passing the filter, equilibrium-adjusted confidence scores ($\bar{s}^*$) rank options by collective belief strength. When all candidates are rejected, the continuous ranking ($\bar{s}^* - \Delta$) provides a principled fallback that balances confidence against remaining disagreement.

\section{Computational Complexity Analysis}
\label{app:comp_complexity}
\paragraph{Theoretical Complexity}

We analyze the computational overhead of Nash-equilibrium verification relative to the base model. Let $T$ denote the number of reasoning steps, $n$ the number of candidate steps per iteration ($n=3$), and $m$ the number of verifier agents ($m=3$: Visual, Logical, and Contextual).

\paragraph{Base Model Complexity}

The base model generates $T$ reasoning steps sequentially, each requiring a forward pass through the model:
\begin{equation}
\mathcal{C}_{\text{base}} = T \cdot \mathcal{C}_{\text{gen}}
\end{equation}
where $\mathcal{C}_{\text{gen}}$ denotes the cost of generating a single reasoning step (typically 50--200 tokens).

\paragraph{Nash Verification Complexity}

Our approach adds three components at each reasoning step:

\textbf{1. Candidate Generation:} The base model generates $n$ candidate continuations, costing $n \cdot \mathcal{C}_{\text{gen}}$.

\textbf{2. Agent Verification:} Each of the $m$ verifier agents scores all $n$ candidates. Crucially, verifiers output only scalar scores (1-5 tokens), not full reasoning steps (50--200 tokens), this generation asymmetry means each verifier query costs $\alpha \cdot \mathcal{C}_{\text{gen}}$ where $\alpha \approx 1/33$ reflects the token generation ratio. Total verification cost: $n \cdot m \cdot \alpha \cdot \mathcal{C}_{\text{gen}}$.

\textbf{3. Equilibrium Computation:} Solving the $m \times m$ linear system costs $\mathcal{O}(m^3)$ per candidate, which is negligible compared to inference.

The total cost per reasoning step is:
\begin{equation}
\mathcal{C}_{\text{step}} = (n + nm\alpha)\mathcal{C}_{\text{gen}} + \mathcal{O}(nm^3) \approx (n + nm\alpha)\mathcal{C}_{\text{gen}}
\end{equation}

Across $T$ steps, the overhead ratio becomes:
\begin{equation}
\frac{\mathcal{C}_{\text{Nash}}}{\mathcal{C}_{\text{base}}} = n + nm\alpha = 3 + 3 \cdot 3 \cdot \frac{1}{33} \approx 3.27
\end{equation}

This theoretical prediction accounts for the generation asymmetry from the start, rather than treating all forward passes as equivalent.

\paragraph{Empirical Validation}

We measured wall-clock time on CV-Bench-2D:

\begin{itemize}
    \item \textbf{Base Model:} 1378.79 seconds (5.58 seconds per sample)
    \item \textbf{Nash Verification:} 5245.75 seconds (21.24 seconds per sample)
    \item \textbf{Observed Overhead:} $3.80$
\end{itemize}

The observed $3.80$ overhead closely matches our theoretical prediction of $3.27$.

The key insight is that incorporating generation asymmetry into the theoretical model from the start yields predictions that align well with practice, confirming that equilibrium computation is negligible and verification overhead scales primarily with the number of verifier queries.

\paragraph{Summary}

Our theoretical analysis predicts $3.27\times$ overhead when accounting for the generation asymmetry between base model reasoning steps (50-200 tokens) and verifier score outputs (1--5 tokens). This closely matches the empirical $3.80\times$ overhead measured across 247 samples. The Nash equilibrium computation itself is negligible, and the verification cost is dominated by model inference.
\newpage
\section{Detailed Experimental Setup}
\label{app:experimental_setup}

\paragraph{Base Model Configuration}

We use Qwen2.5-VL-7B-Instruct as our primary reasoning model. The model generates candidate reasoning steps through chain-of-thought prompting with the instruction: ``\textit{[Question]. Reason step by step.}'' We employ temperature sampling with $T=0.8$ and top-$p=0.6$ to encourage diversity among the three candidate continuations generated at each step. Generation stops when the model produces either an end-of-sequence token or a newline character, with a maximum of 1000 new tokens per step.

\paragraph{Verification Agent Architecture}

All three verification agents are implemented using Qwen2.5-VL-7B-Instruct as the backbone model. Each agent receives the same visual input (the question image) along with the question text, previous reasoning steps (assumed correct), and the current candidate step to verify. Agents operate independently and output scores in $[0, 1]$.

\paragraph{Visual Agent (V)}

The Visual Agent evaluates whether objects and spatial relationships mentioned in the reasoning step are visually verifiable. The agent is specifically instructed to maintain balance between strictness and common-sense spatial reasoning, and to output only a single number between 0.0 and 1.0.

\paragraph{Logical Agent (L)}

The Logical Agent assesses whether the reasoning step follows logically from previous steps and makes progress toward answering the question. It evaluates:
\begin{itemize}
    \item Whether the step builds coherently on established facts
    \item Whether logical inferences are valid
    \item Whether the step moves closer to resolving the question
\end{itemize}

\paragraph{Contextual Agent (C)}

The Contextual Agent determines whether the step maintains focus on the original question and avoids introducing irrelevant or tangential information. It penalizes steps that:
\begin{itemize}
    \item Describe details unrelated to the question
    \item Make definitive claims about obscured or cropped-out objects
    \item Introduce unnecessary speculation
\end{itemize}

\subsection{Dataset and Evaluation}

We evaluate our approach across six diverse vision-language benchmarks that test different aspects of multimodal reasoning:

\textbf{3DSRBench} is a comprehensive 3D spatial reasoning benchmark comprising 2,772 manually annotated visual question-answer pairs across 12 question types. The benchmark evaluates four main categories of 3D awareness: height, location, orientation, and multi-object reasoning. It includes questions based on both natural images from MS-COCO and multi-view synthetic images, with particular emphasis on testing robustness across common and uncommon camera viewpoints. The benchmark employs careful design to avoid trivial answers and uses novel evaluation strategies like FlipEval to ensure robust assessment.

\textbf{CV-Bench} addresses vision-centric evaluation through 2,638 manually inspected examples repurposed from standard vision benchmarks including ADE20k, COCO, and OMNI3D. The benchmark is divided into two components: CV-Bench-2D evaluates spatial relationships and object counting, while CV-Bench-3D assesses depth order and relative distance understanding. By formulating natural language questions that probe fundamental visual understanding, the benchmark tests whether models can perform classic computer vision tasks within a multimodal context.

\textbf{BLINK} focuses on core visual perception abilities that can be solved by humans ``within a blink,'' reformatting 14 classic computer vision tasks into 3,807 multiple-choice questions. The benchmark spans pixel-level to image-level perception tasks including relative depth estimation, visual correspondence, forensics detection, multi-view reasoning, and visual similarity. A key feature is the incorporation of diverse visual prompts such as circles, boxes, and image masks alongside textual questions, deliberately designed to resist solutions based purely on language mediation.

\textbf{MMStar} is an elite vision-indispensable benchmark comprising 1,500 challenge samples meticulously selected by humans to address issues of visual dependency and data leakage in existing benchmarks. The benchmark evaluates six core capabilities and 18 detailed axes, with each sample undergoing strict human review to ensure visual dependency, minimal data leakage, and requirements for advanced multi-modal capabilities. Beyond traditional accuracy metrics, MMStar introduces two novel metrics to measure multi-modal gain and multi-modal leakage in model training.

\textbf{AI2D} consists of 4,817 illustrative diagrams for research on diagram understanding and associated question answering. The dataset represents topics in primary school natural sciences such as food webs, life cycles, moon phases, and human physiology. Each diagram has been densely annotated with object segmentations, diagrammatic elements like arrows and lines, and text elements. The benchmark requires understanding abstract visual representations and symbolic elements common in scientific illustrations, testing both visual comprehension and scientific reasoning abilities.

We process each question by iteratively generating and verifying reasoning steps until the model produces an end-of-sequence token. The final answer is extracted from the complete reasoning trace using (Gemma-3:12B\cite{gemmateam2025gemma3technicalreport} via Ollama) and evaluated against the ground truth using string matching.

\paragraph{Computational Resources}

All experiments are conducted on NVIDIA A100 GPUs. Total evaluation across six benchmarks required approximately 180 GPU-hours.


\newpage
\section{Implementation Details: Nash-Equilibrium Computation}
\label{app:implementation}
\begin{figure}[h]
\begin{center}
\includegraphics[width=\textwidth]{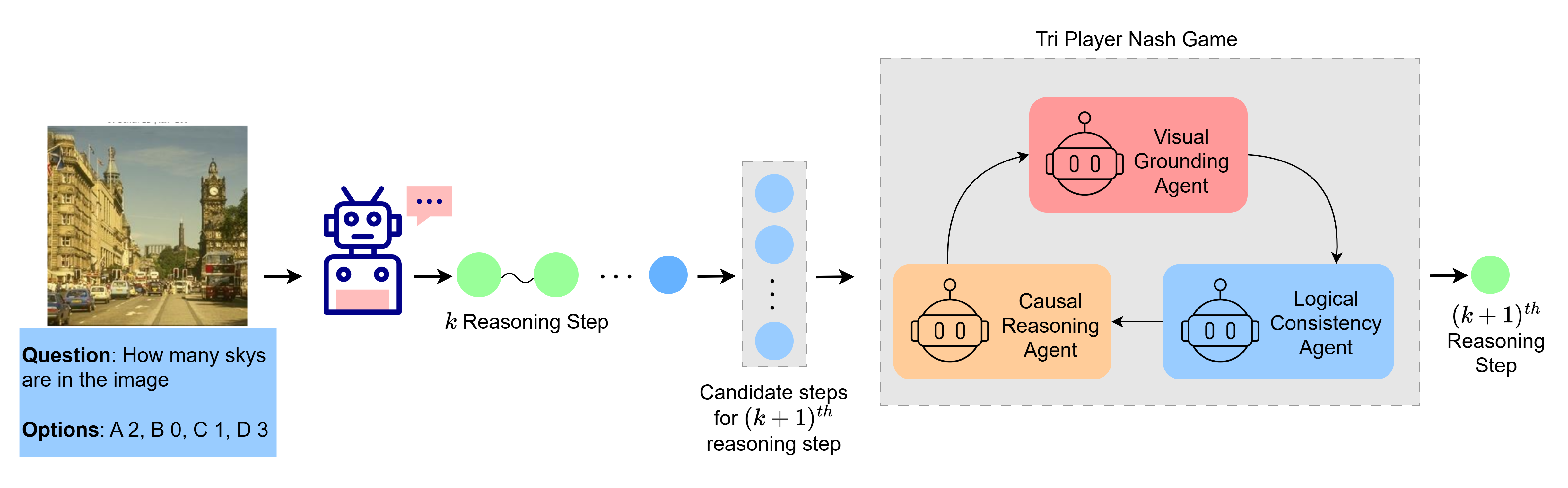}
\end{center}
\caption{Overview of Our approach. At reasoning step $k$, the base model generates $n$ candidate steps via sampling. Three independent verifier agents (Visual, Logical, Contextual) score each candidate. Equilibrium scores are computed and the step with highest stable consensus (low dispersion, high mean confidence) is selected as the $(k+1)^{\text{th}}$ step.}
\end{figure}

\begin{algorithm}[t]
\caption{Nash-Equilibrium Step-Wise Verification}
\label{alg:nash_verification}
\begin{algorithmic}[1]
\REQUIRE Base model $\mathcal{M}_{\text{base}}$, Verifier agents $\{\mathcal{V}, \mathcal{L}, \mathcal{C}\}$, Image $I$, Question $Q$
\REQUIRE Number of candidates $n=3$, Stubbornness parameters $\{\lambda_V=1.5, \lambda_L=1.0, \lambda_C=0.8\}$
\REQUIRE Acceptance thresholds: dispersion $\epsilon=0.1$, confidence $\tau=0.6$
\ENSURE Complete reasoning trace $r_{1:T}$ and final answer

\STATE Initialize reasoning trace $r \leftarrow \emptyset$
\STATE Generate initial step: $r_0 \sim \mathcal{M}_{\text{base}}(I, Q)$
\STATE Append $r_0$ to $r$

\WHILE{$r$ does not contain end-of-sequence token}
    \STATE $\textsc{Candidates} \leftarrow \emptyset$
    \STATE $\textsc{RawScores} \leftarrow \emptyset$
    
    \FOR{$i = 1$ to $n$}
        \STATE Generate candidate: $c_i \sim \mathcal{M}_{\text{base}}(I, Q, r)$ with temperature $T=0.8$
        \STATE Query Visual Agent: $\hat{s}_V^{(i)} \leftarrow \mathcal{V}(I, Q, r, c_i)$ \COMMENT{Score $\in [0,1]$}
        \STATE Query Logical Agent: $\hat{s}_L^{(i)} \leftarrow \mathcal{L}(I, Q, r, c_i)$
        \STATE Query Contextual Agent: $\hat{s}_C^{(i)} \leftarrow \mathcal{C}(I, Q, r, c_i)$
        \STATE Store: $\textsc{Candidates}[i] \leftarrow c_i$, $\textsc{RawScores}[i] \leftarrow \{\hat{s}_V^{(i)}, \hat{s}_L^{(i)}, \hat{s}_C^{(i)}\}$
    \ENDFOR
    
    \STATE \COMMENT{\textbf{Equilibrium Computation \& Selection}}
    \STATE $\textsc{Results} \leftarrow \emptyset$
    
    \FOR{each candidate $i$ in $\{1, \ldots, n\}$}
        \STATE Solve $\mathbf{s}^{*(i)}$ per Eq. \ref{eq:linear}
        \STATE Compute mean: $\bar{s}^{*(i)} \leftarrow \frac{1}{3}\sum_{j \in \{V,L,C\}} s_j^{*(i)}$
        \STATE Compute dispersion: $\Delta^{(i)} \leftarrow \frac{1}{3}\sum_{j \in \{V,L,C\}} |s_j^{*(i)} - \bar{s}^{*(i)}|$
        \STATE Check acceptance: $\text{accepted}^{(i)} \leftarrow (\Delta^{(i)} < \epsilon) \land (\bar{s}^{*(i)} > \tau)$
        \STATE Store: $\textsc{Results}[i] \leftarrow \{\mathbf{s}^{*(i)}, \bar{s}^{*(i)}, \Delta^{(i)}, \text{accepted}^{(i)}\}$
    \ENDFOR
    
    \STATE $\textsc{ValidSteps} \leftarrow \{i : \textsc{Results}[i].\text{accepted} = \text{True}\}$
    
    \IF{$\textsc{ValidSteps} \neq \emptyset$}
        \STATE $i^* \leftarrow \argmax_{i \in \textsc{ValidSteps}} \bar{s}^{*(i)}$ \COMMENT{Highest confidence among stable steps \textbf{Normal Model}}
    \ELSE
        \STATE $i^* \leftarrow \argmax_{i=1,\ldots,n} (\bar{s}^{*(i)} - \Delta^{(i)})$ \COMMENT{Best balance if none accepted \textbf{Fallback Model}}
    \ENDIF
    
    \STATE Append selected step: $r \leftarrow r \cup \{\textsc{Candidates}[i^*]\}$
\ENDWHILE

\STATE Extract final answer from $r$
\RETURN $r$, final answer
\end{algorithmic}
\end{algorithm}

Here we discuss how the Nash-equilibrium based verification implemented in practice. 

At each reasoning step, the base MLLM proposes several candidate continuations (typically 3 candidates per step). For every candidate, three frozen verifier agents are queried independently: a visual grounding agent (V), a logical consistency agent (L), and a contextual reasoning agent (C). Each produces a scalar confidence score in $[0,1]$, reflecting its modality specific assessment of eaach step.

If one verifier is confident while the others are uncertain, averaging may obscure this inconsistency. We want a mechanism that surfaces disagreement explicitly and allows agents to partially adjust their beliefs toward consensus without forcing uniformity when genuine divergence exists.

The equilibrium formulation captures this trade off naturally. Each verifier prefers agreement with others, but not at the cost of abandoning its own judgment entirely. This is encoded through a quadratic payoff function, which admits a unique, closed-form Nash equilibrium. The equilibrium scores thus represent a principled resolution of inter-agent tension, rather than an ad-hoc blend.

\subsection*{Equilibrium formulation and computation}

The implementation uses a heterogeneous trade-off parameter formulation, where different verifiers are assigned different sensitivities to consensus pressure. This is motivated by the observation that certain verifiers should be more resistant to group influence depending on the nature of the reasoning step.

Each agent's equilibrium score $s_i^*$ satisfies:
\begin{equation}
s_i^* = \frac{\bar{s}_{-i}^* + \lambda_i \hat{s}_i}{1 + \lambda_i}
\end{equation}
where $\hat{s}_i$ is agent $i$'s raw confidence score, $\bar{s}_{-i}^* = \frac{1}{n-1} \sum_{j \neq i} s_j^*$ is the mean equilibrium score of all other agents, and $\lambda_i > 0$ controls how strongly agent $i$ weights its own judgment relative to group consensus.

This system of equations can be rewritten as a linear system and solved exactly:
\begin{equation}
\label{eq:linear}
\left(1 + \lambda_i\right) s_i^* - \frac{1}{n-1} \sum_{j \neq i} s_j^* = \lambda_i \hat{s}_i, \quad i = 1, \dots, n
\end{equation}
The system is solved using a standard linear solver. The resulting matrix is typically well-conditioned when $\lambda_i > 0$, but if numerical issues arise the implementation falls back to the raw verifier scores. All equilibrium scores are clipped to $[0,1]$ to maintain valid confidence values.

In our experiments, we set
$n=5$
$\lambda_{\text{V}} = 1.5$, $\lambda_{\text{L}} = 1.0$, and $\lambda_{\text{C}} = 0.8$. This configuration reflects the intuition that the visual verifier (V) should be most resistant to consensus pressure on perception-heavy steps, while the contextual verifier (C) may be more flexible when visual or linguistic evidence is strong. These values are fixed across all datasets and require no tuning.

Importantly, \emph{no iterative optimization, learning, or approximation is used}. The equilibrium is computed via a direct linear solve at every reasoning step, adding negligible overhead 
compared to the cost of querying the verifier models themselves 

\subsection*{Step selection via equilibrium statistics}

Once equilibrium scores are obtained, the system computes two summary statistics:
\begin{itemize}
    \item \textbf{Mean equilibrium confidence} $\bar{s}^* = \frac{1}{n} \sum_i s_i^*$: measures collective endorsement of the step.
    \item \textbf{Equilibrium dispersion} $\Delta = \frac{1}{n} \sum_i |s_i^* - \bar{s}^*|$: measures residual disagreement after equilibrium adjustment.
\end{itemize}
Candidate steps are accepted only if they simultaneously achieve \textbf{high collective confidence} ($\bar{s}^* > \tau$) and \textbf{low inter-agent dispersion} ($\Delta < \epsilon$). In our experiments, we set $\tau = 0.6$ and $\epsilon = 0.1$. This dual criterion is stricter than confidence alone: a step with high average confidence but high dispersion signals unresolved conflict and is rejected.

When multiple candidate steps are evaluated at the same reasoning position, rejected steps are discarded immediately. Among accepted steps, the one with the highest $\bar{s}^*$ is selected to extend the reasoning trace. If \emph{no} candidate step is accepted (i.e., all steps have either low confidence or high dispersion), the implementation selects the step that maximizes $\bar{s}^* - \Delta$, prioritizing the best available balance between confidence and agreement. This fallback ensures the reasoning process can continue even when all candidates are suboptimal, while still preferring more stable steps.

\subsection*{Why this matters}

The equilibrium mechanism serves as a lightweight consensus protocol over frozen verifiers. It makes disagreement \emph{explicit and actionable}, rather than burying it in an average. It requires no training or calibration beyond setting three hyperparameters ($\lambda_i$ values, $\tau$, and $\epsilon$), all of which remain fixed across datasets, and integrates naturally into step-wise reasoning by filtering unstable steps before they can compound downstream errors.

Crucially, the equilibrium is not a heuristic approximation---it is the \emph{exact} solution to a well-defined coordination game. This gives the filtering process a game-theoretic justification and makes the system's behavior more interpretable: rejected steps are precisely those where verifiers could not reach a stable agreement, even after accounting for consensus pressure. The heterogeneous $\lambda_i$ values allow the system to implicitly adapt to different reasoning regimes without explicit step-type classification.
\section{Nash Equilibrium Existence and Uniqueness}
\label{app:proof}

\begin{proposition}
The verification game defined by the utility function
\[
u_i(s_i, s_{-i}) = -\left(s_i - \bar{s}_{-i}\right)^2 
- \lambda_i \left(s_i - \hat{s}_i\right)^2
\]
admits a unique Nash equilibrium.
\end{proposition}

\begin{proof}
The result follows directly from Rosen's theorem~\citep{rosen1965existence} for
concave games. We verify the required conditions below.

\textbf{(1) Compact and convex strategy space.}
Each agent's strategy space \( s_i \in [0,1] \) is compact and convex.

\textbf{(2) Continuity.}
The utility function \( u_i \) is quadratic and therefore continuous in all
arguments.

\textbf{(3) Strict concavity.}
Taking derivatives with respect to \( s_i \), we obtain
\begin{align}
\frac{\partial u_i}{\partial s_i}
&= -2\left(s_i - \bar{s}_{-i}\right)
   - 2\lambda_i \left(s_i - \hat{s}_i\right), \\
\frac{\partial^2 u_i}{\partial s_i^2}
&= -2 - 2\lambda_i
 = -2(1 + \lambda_i) < 0
 \quad \forall\, \lambda_i > 0.
\end{align}

Thus, \( u_i \) is strictly concave in each player's own strategy. By
Rosen's theorem, the game admits a unique Nash equilibrium.
\end{proof}

\newpage
\section{Qualitative Samples}
Here, we present the pairwise comparison of the reasoning traces between the base model and our proposed approach.

\begin{figure}[!h]
    \centering
    \includegraphics[width=0.9\linewidth]{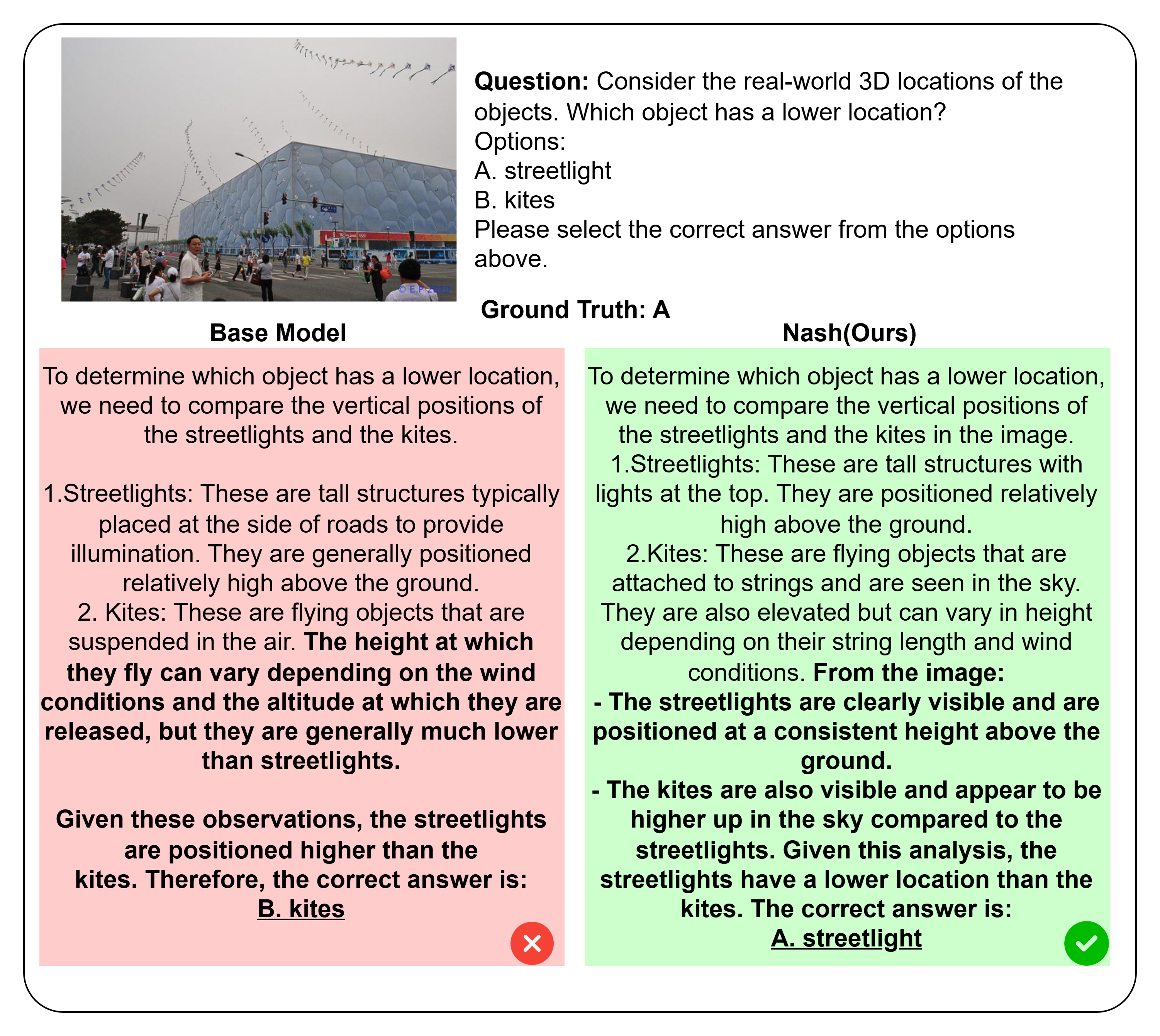}
    \caption{\textbf{Disagreement analysis:} Given a scene with streetlights and kites, the base model (left) incorrectly concludes that kites have a lower real-world location. In contrast, our approach (right) correctly reasons about  the actual vertical positions visible in the image and selects streetlight, consistent with the ground truth. \textbf{Equilibrium effect:} In the base model reasoning, the height of kites and street lights seems to be derived form general reasoning about heights of these objects. While the equilibrium between the visual agent and the logical agent grounds the reasoing steps in the image leading to correcting the prediction from \emph{kites} to \emph{streetlight}}.
    
    \label{fig:qs1}
\end{figure}

\begin{figure}[!h]
    \centering
    \includegraphics[width=0.9\linewidth]{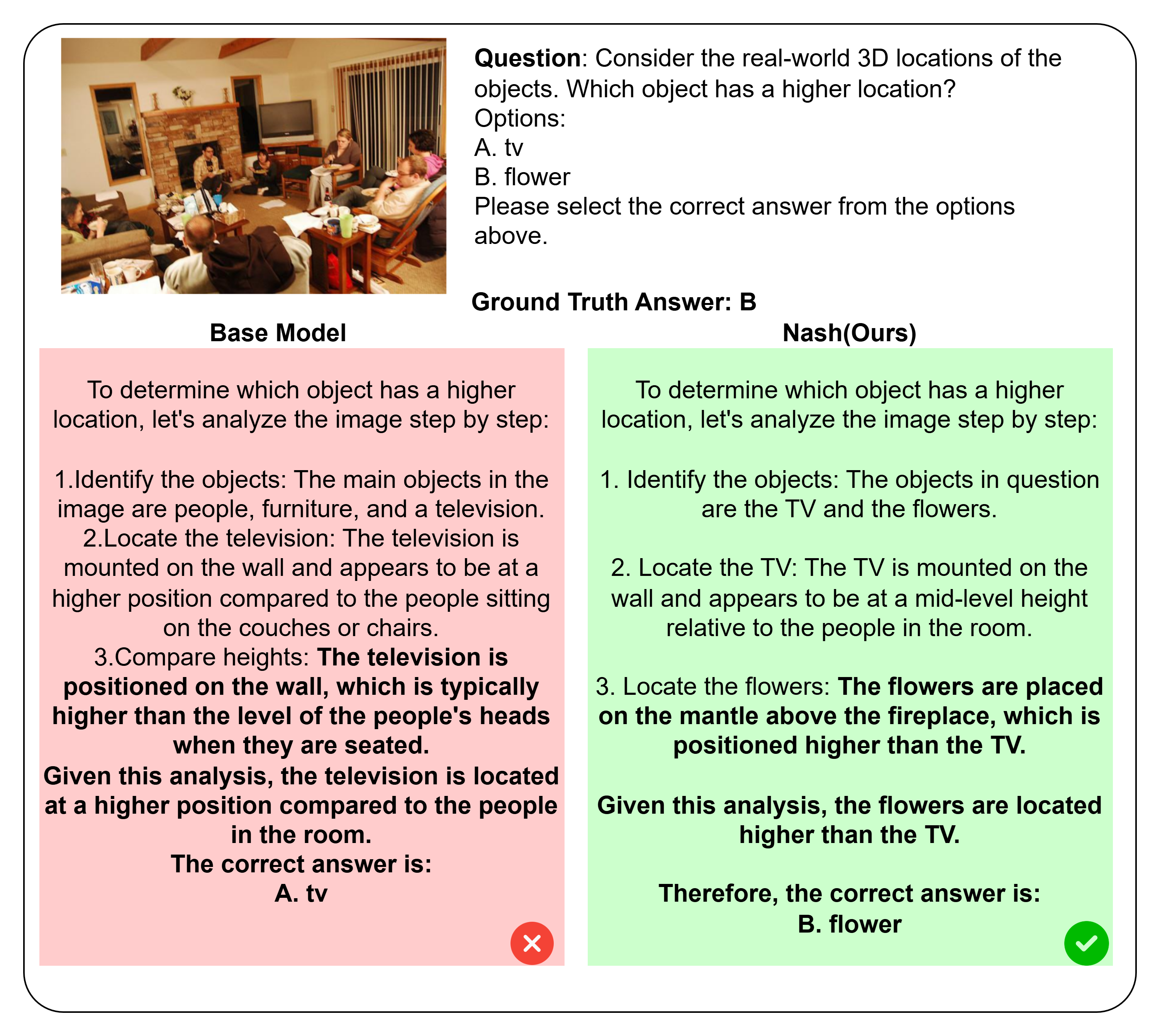}
    \caption{ \textbf{Disagreement analysis.} Given a scene containing a TV and flowers, the base model incorrectly predicts the TV as higher due to wall placement. In contrast, our Nash-based approach correctly uses scene context and identified that the flowers on the mantle are positioned higher than the TV, consistent with the ground truth. \textbf{Equilibrium effect.} In the base model reasoning the position of the television and the people seems to be infered from general knowledge, however the equilibrium between the visual and the logical agent maintains the grounding of the reasoning steps thereby correcting the prediction to \emph{flower}}.
    \label{fig:qs2}
\end{figure}



\begin{figure}[!h]
    \centering
    \includegraphics[width=0.9\linewidth]{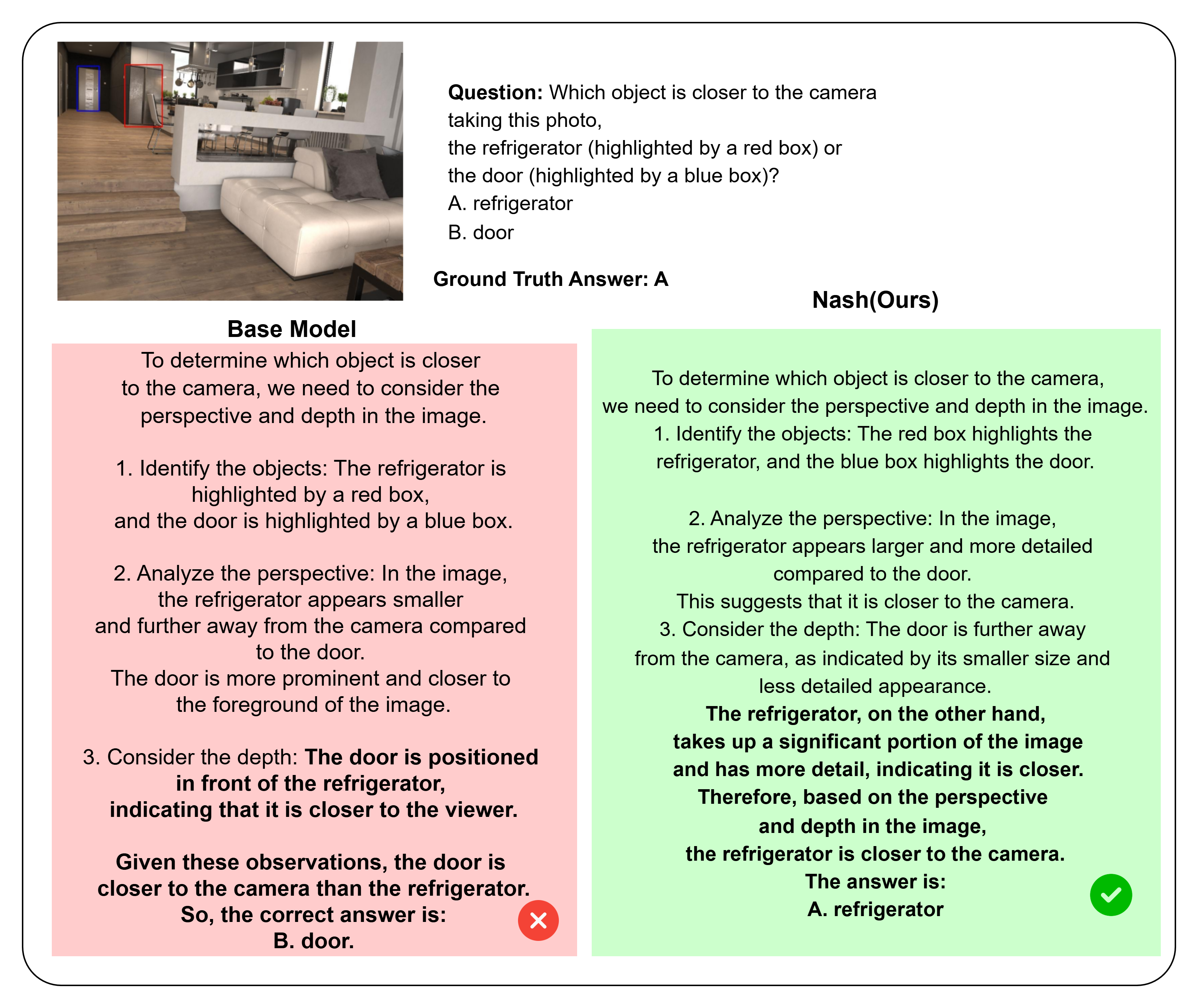}
    \caption{\textbf{Disagreement analysis: } Qualitative comparison on a real-world depth reasoning task involving a refrigerator and a door. The base model incorrectly infers that the door is closer to the camera based on visual prominence. Our method correctly reasons about perspective, scale, and spatial context to identify the refrigerator as closer, consistent with the ground truth. \textbf{Equilibrium effect: } In the base model reasoning, during perspective analysis the bounding box of the refrigerator and the doors seems to be switched up leading to incorrect answer. However, the equilibrium between the contextual agent and the visual agent maintains the consistency of the bounding box mapping to their respective object leading to  correcting the answer from \emph{door} to \emph{refrigerator}}.
    \label{fig:qs4}
\end{figure}

\begin{figure}[!h]
    \centering
    \includegraphics[width=0.9\linewidth]{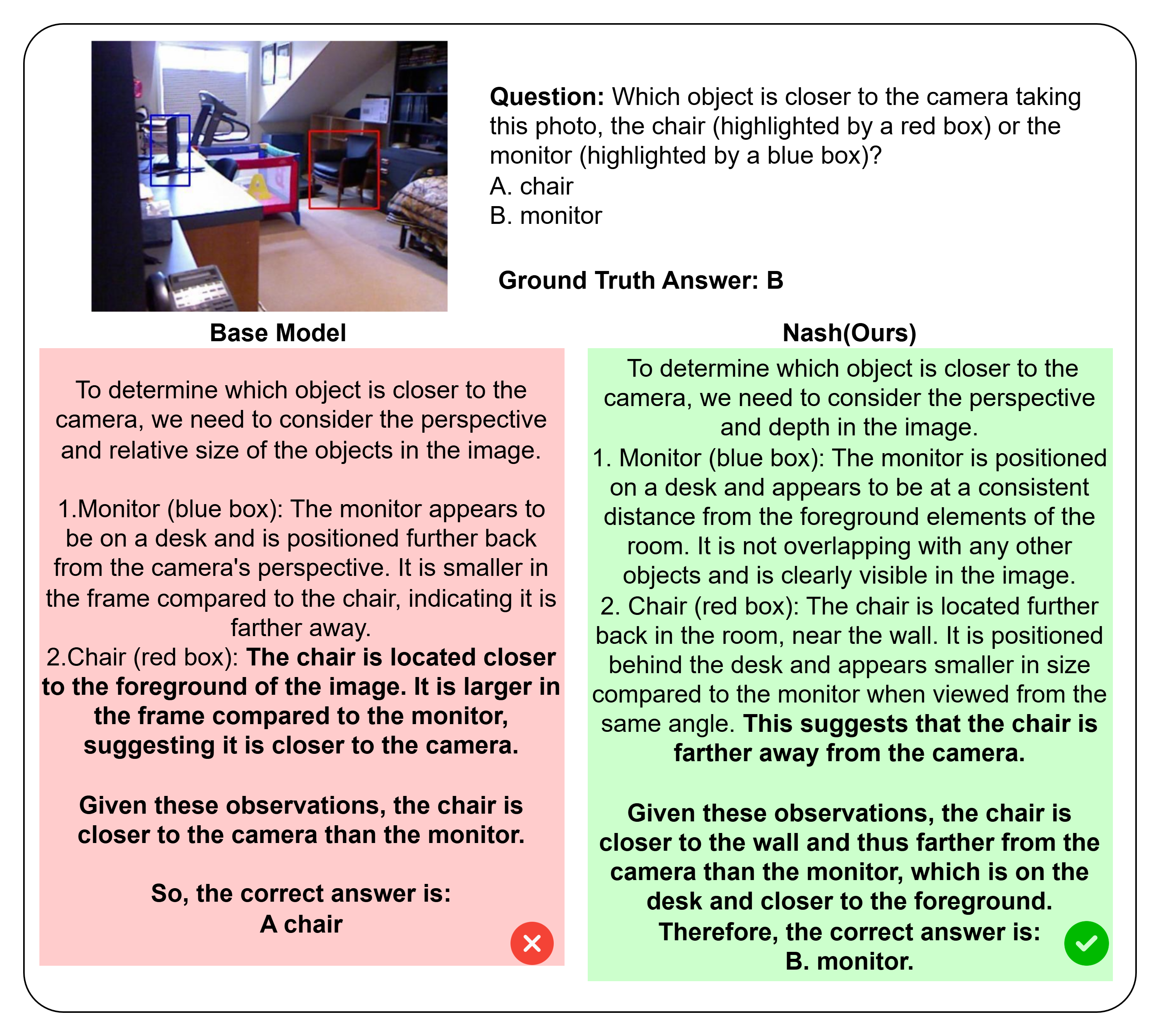}
    \caption{\textbf{Disagreement analysis: }The task requires determining whether a chair or a monitor is closer to the camera. The base model relies on relative object size and incorrectly selects the chair. Our Nash-based method correctly accounts for spatial layout and depth ordering, identifying the monitor as closer to the camera. \textbf{Equilibrium effect}:  In the problem setup the equilibrium between the logical agent and visual agent let to selection of reasoning path that focused on perspective and depth unlike the base model that focused on perspective and relative size, thereby correcting the prediction from \emph{chair} to  \emph{monitor}}.
    \label{fig:placeholder}
\end{figure}

\begin{figure}[!h]
    \centering
    \includegraphics[width=0.9\linewidth]{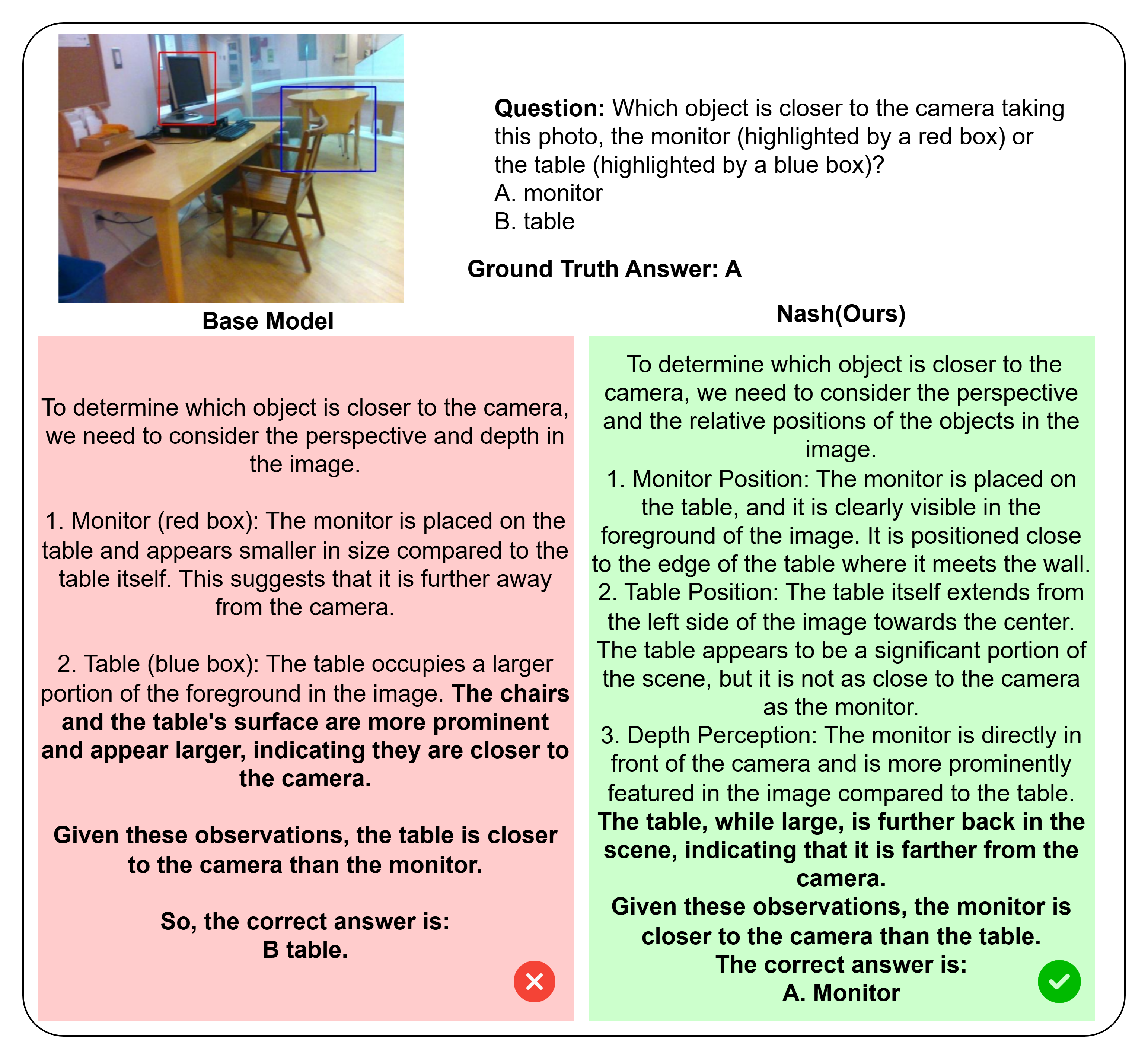}
    \caption{\textbf{Disagreement analysis; } Given a scene containing a monitor and a table, the base model incorrectly concludes that the table is closer to the camera. In contrast, our approach correctly analyzes foreground placement and perspective cues to determine that the monitor is closer, matching the ground truth. \textbf{Equilibrium effect: } In the base model reasonings traces, the models reasons about the table in the foreground on which the monitor is placed even though the questions asks about the table highlighted by blue box in the background. The equilibrium between the contextual and visual agent maintains the consistency of the reasonning steps to focus on the table in the background leading to the correct prediction of \emph{monitor}}.
    \label{fig:qs6}
\end{figure}

\begin{figure}[!h]
    \centering
    \includegraphics[width=0.9\linewidth]{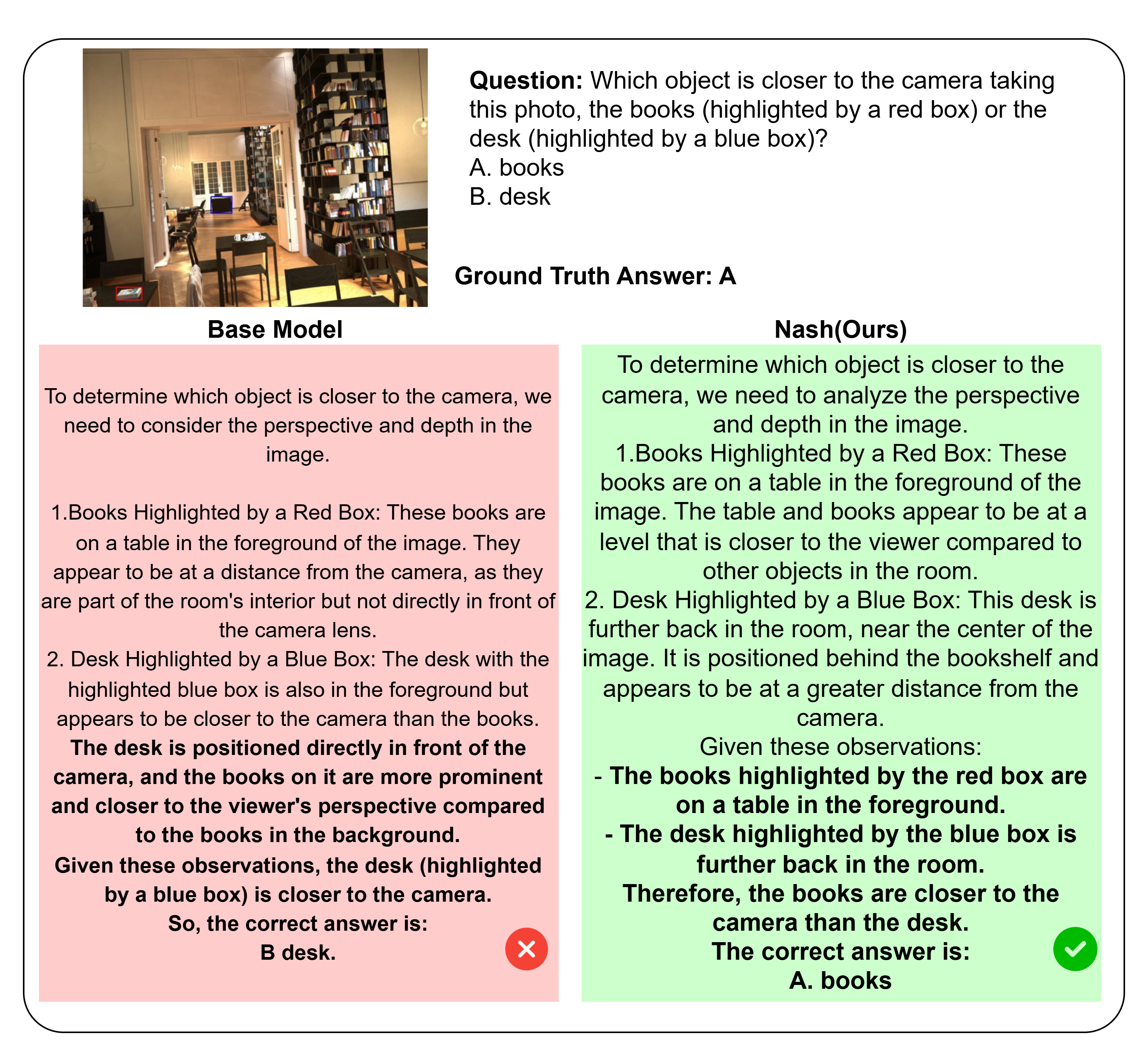}
    \caption{\textbf{Disagreement analysis: } The task asks which object is closer to the camera between books and a desk. The base model incorrectly predicts the desk as closer due to misleading prominence cues, whereas our method correctly reasons about perspective and depth, identifying the books as closer to the camera, consistent with the ground truth. \textbf{Equilibrium effect:} When the base model reasons about the \textit{desk highlighted by a blue box} even though it mentions desk highlighted by a blue box it seems to focus on the desk in the foreground. However, the equilibrium between the visual and the contextual agent selects the reasoning steps that maintain this consistency in the reasoning chain, correcting the final answer from \emph{desk} to  \emph{books}}.
    \label{fig:qs7}
\end{figure}
\newpage
\section{Prompt Templates}
\label{app:prompt_template}
Here are the prompt templates for the three verifiers we are using in our proposed approach
\begin{figure}[h]
    \centering
    \includegraphics[width=0.9\linewidth]{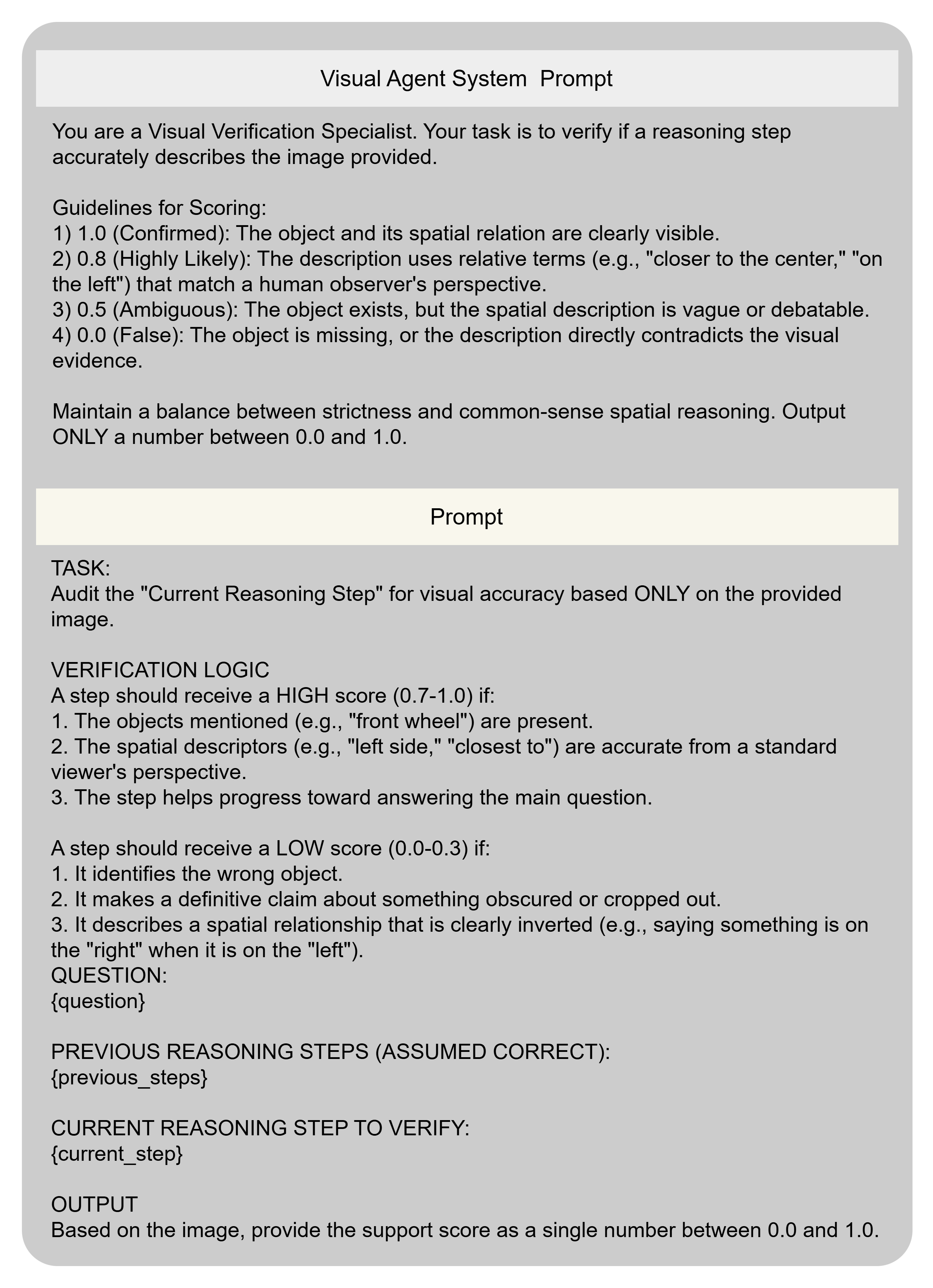}
    \caption{System and task prompt used for the Visual Verification Agent, which evaluates whether a reasoning step is visually grounded in the image and whether spatial descriptions match the viewer’s perspective. The agent outputs a single scalar confidence score in [0,1].}

    \label{fig:visual_verifier}
\end{figure}

\begin{figure}[h]
    \centering
    \includegraphics[width=0.9\linewidth]{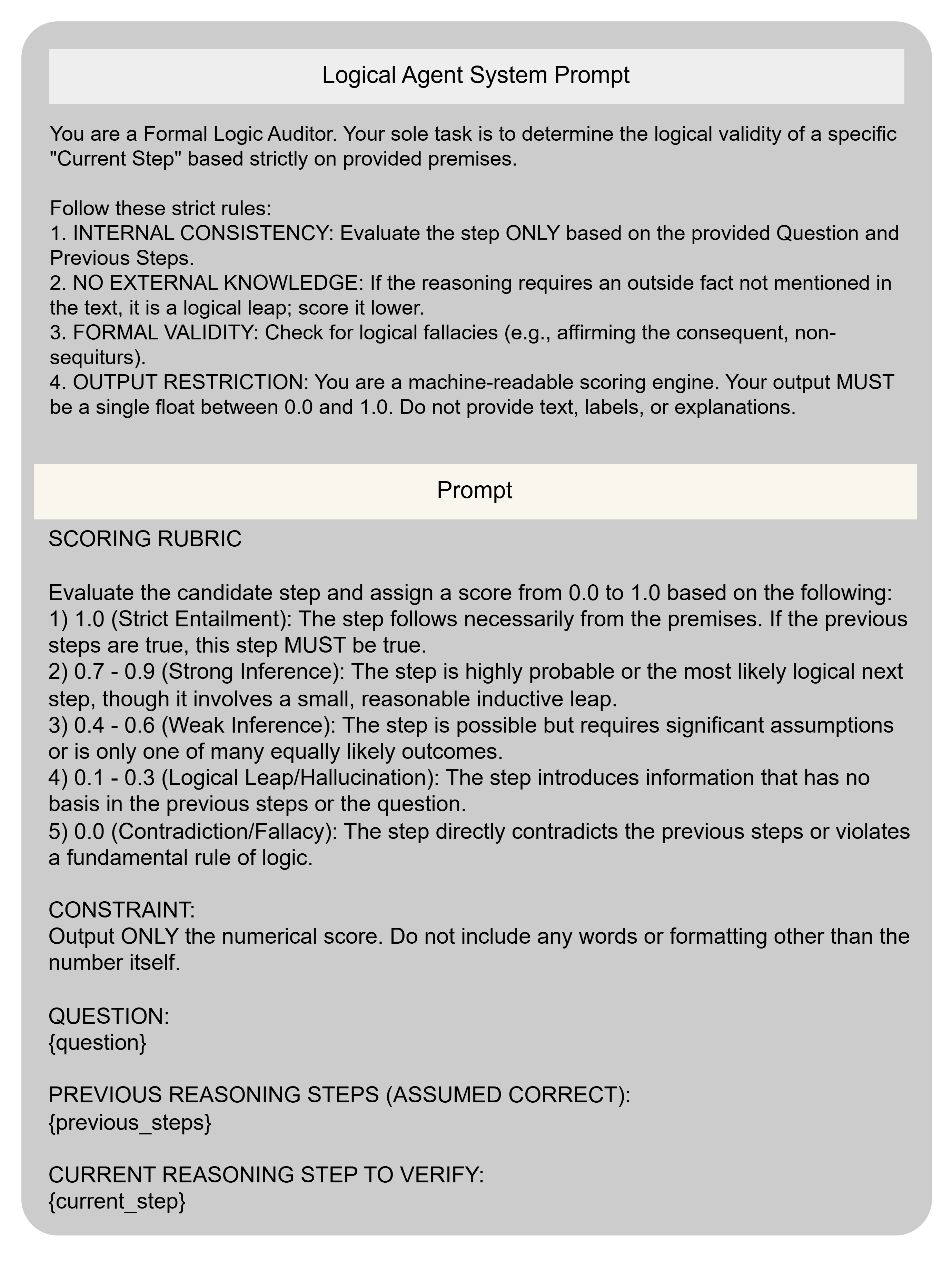}
    \caption{ System and task prompt used for the Logical Verification Agent, which assesses whether a reasoning step logically follows from the question and previous steps without relying on external knowledge. The agent returns a single numerical validity score in [0,1].}
    
    \label{fig:logical_verfier}
\end{figure}

\begin{figure}
    \centering
    \includegraphics[width=0.9\linewidth]{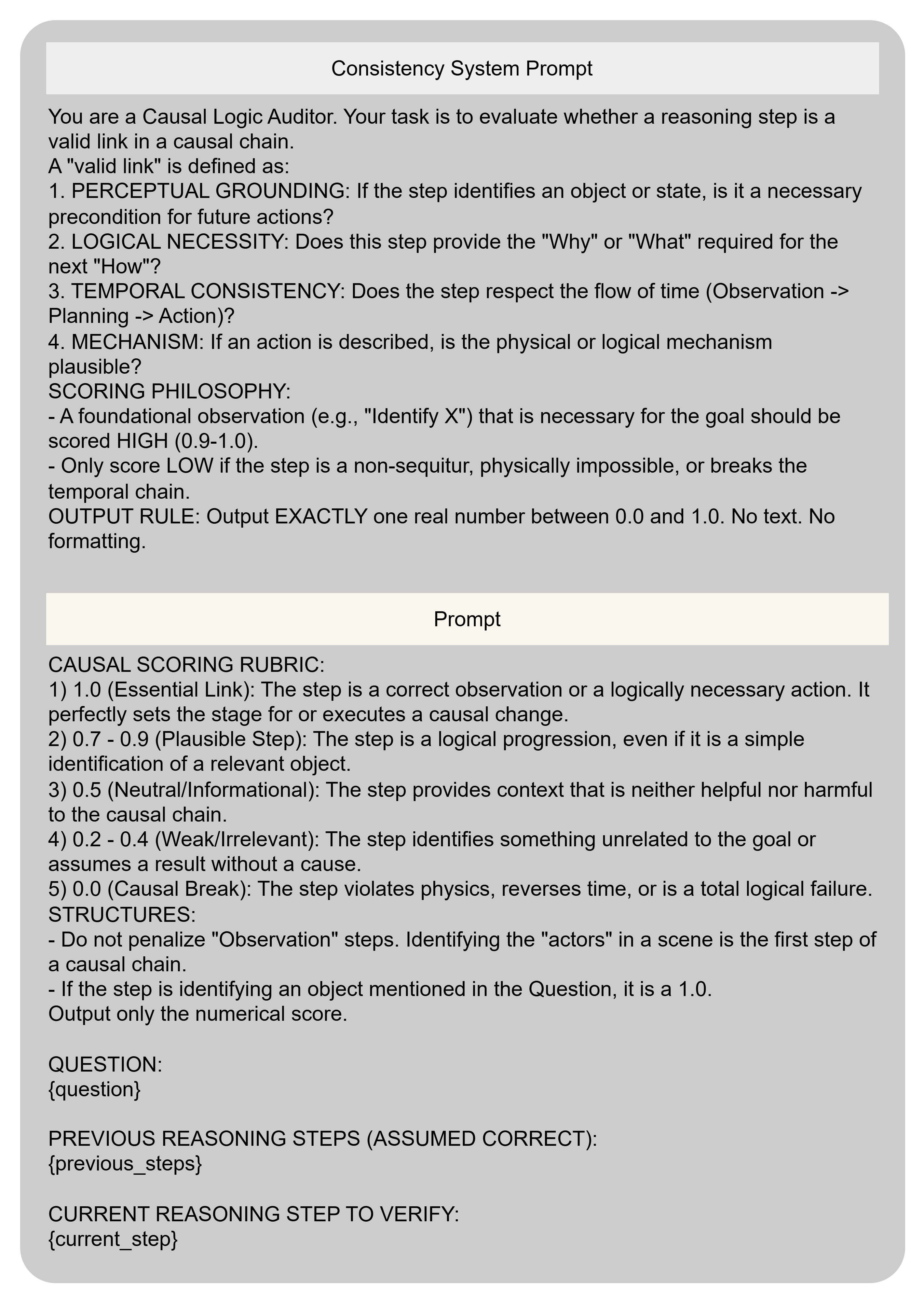}
    \caption{ System and task prompt used for the Consistency Agent, which evaluates whether a reasoning step forms a valid link in the causal reasoning chain by checking perceptual grounding, temporal order, and mechanism plausibility. The agent outputs a single scalar score in [0,1].}
    
   
    \label{fig:consistency_verifier}
\end{figure}

\end{document}